%% file: FedU_Arvix.tex
\documentclass[journal, 10]{IEEEtran}
\IEEEoverridecommandlockouts

\ifCLASSINFOpdf
\usepackage[pdftex]{graphicx}
\graphicspath{{/figs/}}
\DeclareGraphicsExtensions{.pdf,.eps}
\else
\usepackage[dvips]{graphicx}
\graphicspath{{/figs/}}
\DeclareGraphicsExtensions{.eps}
\fi

\usepackage{caption}
\usepackage{microtype}
\usepackage{subcaption}
\usepackage{booktabs} 
\usepackage{cite,algorithm,algorithmic,amsmath,amssymb,amsthm,empheq,mhsetup,balance}
\usepackage{multirow}
\usepackage{xr}
\usepackage{enumitem}
\usepackage{hyperref}
\usepackage{xcolor}
\usepackage{cleveref}
\usepackage{xspace}
\usepackage[compact]{titlesec}

\renewcommand{\algorithmicwhile}{\textbf{on client}}
\renewcommand{\algorithmicendwhile}{\textbf{end on client}}

\DeclareMathOperator{\Xt}{\widetilde{X}}
\DeclareMathOperator{\mut}{\widetilde{\mu}}

\DeclareMathOperator{\tr}{\mathrm{tr}}

\DeclareMathOperator{\EEE}{\mathbb{E}}

\DeclareMathOperator{\GG}{\mathcal{G}}

\DeclareMathOperator{\OO}{\mathcal{O}}
\DeclareMathOperator{\EE}{\mathcal{E}}

\DeclareMathOperator{\R}{\mathcal{R}}
\DeclareMathOperator{\RRR}{\mathbb{R}}

\DeclareMathOperator{\LL}{\mathcal{L}}

\DeclareMathOperator{\SSS}{\mathcal{S}}

\DeclareMathOperator{\AAA}{\mathcal{A}}

\DeclareMathOperator{\NN}{\mathcal{N}}
\DeclareMathOperator{\NNN}{\mathbb{N}}

\DeclareMathOperator{\diag}{\mathrm{diag}}

\newtheorem{theorem}{Theorem}
\newtheorem{assumption}{Assumption}
\newtheorem{lemma}{Lemma}
\newtheorem{remark}{Remark}

\newcommand{\OurAlg}{\ensuremath{\textsf{FedU}}\xspace}
\newcommand{\dFed}{\textsf{dFedU}\xspace}

\DeclarePairedDelimiterX{\norm}[1]{\lVert}{\rVert}{#1}
\DeclarePairedDelimiterX{\abs}[1]{\lvert}{\rvert}{#1}
\DeclarePairedDelimiterX{\innProd}[1]{\langle}{\rangle}{#1}

\def\BibTeX{{\rm B\kern-.05em{\sc i\kern-.025em b}\kern-.08em
		T\kern-.1667em\lower.7ex\hbox{E}\kern-.125emX}}
\begin{document}
	\bstctlcite{IEEEexample:BSTcontrol}

	\title{A New Look and Convergence Rate of Federated Multi-Task Learning with Laplacian Regularization}
	\author{
		Canh~T.~Dinh,
		Tung T. Vu\emph{, Member, IEEE}, 
		Nguyen~H.~Tran\emph{, Senior Member, IEEE}, 
		Minh N. Dao,
		\\ Hongyu Zhang\emph{, Senior Member, IEEE}
		\thanks{C.~T.~Dinh and N.~H.~Tran are with the School of Computer Science, The University of Sydney, Sydney, NSW 2006, Australia (email: \{canh.dinh,nguyen.tran\}@sydney.edu.au)}
		\thanks{T.~T.~Vu is with Institute of Electronics, Communications, and Information Technology (ECIT), Queen's University Belfast, Belfast BT3 9DT, UK (e-mail: t.vu@qub.ac.uk).}
		\thanks{M.~N.~Dao is with the School of Engineering, Information Technology and Physical Sciences, Federation University, Ballarat, VIC 3353, Australia (e-mail: m.dao@federation.edu.au).}
		\thanks{H.~Zhang is with the University of Newcastle, Callaghan, NSW 2308, Australia (e-mail: hongyu.zhang@newcastle.edu.au)}

		}
		

		\maketitle
		
		\begin{abstract}
			Non-Independent and Identically Distributed (non-IID) data distribution among clients is considered as the key factor that degrades the performance of federated learning (FL). Several approaches to handle non-IID data such as personalized FL and federated multi-task learning (FMTL) are of great interest to research communities. In this work, first, we formulate the FMTL problem using Laplacian regularization to explicitly leverage the relationships among the models of clients for multi-task learning. Then, we introduce a new view of the FMTL problem, which in the first time shows that the formulated FMTL problem can be used for conventional FL and personalized FL. We also propose two algorithms \OurAlg and \dFed to solve the formulated FMTL problem in communication-centralized and decentralized schemes, respectively. Theoretically, we prove that the convergence rates of both algorithms achieve linear speedup for strongly convex and sublinear speedup of order $1/2$ for nonconvex objectives. Experimentally, we show that our algorithms outperform the conventional algorithm FedAvg, {FedProx, SCAFFOLD, and AFL} in FL settings, MOCHA in FMTL settings, as well as pFedMe and Per-FedAvg in personalized FL settings.
		\end{abstract}
		\begin{IEEEkeywords}
			Federated multi-task learning, federated learning, personalized learning, Laplacian regularization.
		\end{IEEEkeywords}
		
		\vspace{-0mm}
		\input{introduction}

		\label{Intro}
		
		\vspace{-0mm}
		\section{Federated Multi-Task Learning: A New View}
		\vspace{-0mm}
		\subsection{The Formulation of the FMTL Problem with Laplacian Regularization}
		\vspace{-0mm}

		In this work, the goal of FMTL is to fit separate models (i.e., $w_k\in\RRR^d, \forall k\in\NN$) to the local data of clients, taking into account the relationships among these models.
		For instance, smart-device clients in a mobile network are trying to learn their activities using their personal and private data (e.g., image, text, voice, and sensor data). In FL settings, their data may come from different environments, contexts, and applications, and thus, have non-IID distributions.
		Despite of this, these clients are likely to behave similarly under similar features or scenarios 
		(e.g., location, time, age). Therefore, there normally exist relationships 
		among the models of clients \cite{argyriou08,ando05,caruana97}.
		
		{To present the relationships among the models of clients, we consider a connected graph $\GG=\{\NN,\EE,A\}$, where $\NN:=\{1,\dots,N\}$ is the set of vertices representing federated learning clients, $\EE$ is the set of edges representing relationships among the models of clients, and $A\in\RRR^N$ is a symmetric, weighted adjacency matrix with $a_{k\ell} := [A]_{k\ell}$. The relationship between clients $k$ and $\ell$ is presented by $a_{k\ell}$ and reversible, i.e., $a_{k\ell} = a_{\ell k}, \forall k, \ell$. Here,} $a_{k\ell}=0$ means no relationship between the models of clients $k$ and $\ell$. The value of $a_{k\ell} > 0$ shows
		that client $k$ is a neighbor of client $\ell$ and also determines the strength of the relationship between
		these two clients' models.
		Let $D\in\RRR^N$ be a diagonal matrix in which $[D]_{kk}=\sum\nolimits_{\ell=1}^N a_{k\ell}$. The Laplacian matrix of the graph is thus $L=D-A$.
		
		
		Let $W=[w_1^T,...,w_N^T]^T\in\RRR^{dN}$ be a collective model vector and $\LL := L\otimes I_d$ be a Laplacian regularization matrix. Now, we formulate the following FMTL problem:
		\begin{align}\label{mainP:FMTL}
			\underset{W}{\min} \,\,
			J(W)\!=\!\underbrace{F(W)}_{\text{Global loss}}\! +\! \underbrace{\eta R(W)}_{\text{Laplacian regularization}},
		\end{align}
		where 
		\begin{align}
			&F(W)\!=\!
			\sum\nolimits_{k=1}^N\!F_k(w_k),
			\\
			\label{regularization}
			&\R(W) = W^T\LL W=\frac{1}{2}\sum\nolimits_{k=1}^N
			\sum\nolimits_{\ell\in\NN_k}a_{k\ell}||w_k-w_\ell||^2,
		\end{align}
		$\NN_k = \NN\setminus \{k\}$, and $\|\cdot\|$ is the Euclidean norm.
		$F_k(\cdot)$ represents the expected loss function at client $k$:
		\begin{align*}\label{Fk}
			F_k(w_k)=
				\EEE_{\zeta_{k}}[f_{k}(w_k;\zeta_{k})],
		\end{align*}
		where $\zeta_k$ is a random data sample drawn from
		the distribution of client $k$ and  $f_{k}(w_k;\zeta_{k})$ is the regularized loss function corresponding to this sample and $w_k$.
		The distribution of $\zeta_k$ and $\zeta_\ell$ can be distinct when $k\neq \ell$.
		
		{Note that in our work, we do not extract the similarity of the existing relationships between the clients by any visualization methods in order to develop our proposed method. Instead, we present the existing relationships among the models of the clients by a Laplacian regularization matrix $\LL$ and put it into the Laplacian regularization term in the objective function of the federated multitask-learning problem \eqref{mainP:FMTL}. 
		Theoretically, in \eqref{mainP:FMTL}, $\eta\geq 0$ is a regularization hyperparameter that controls the impact of the models of neighboring clients on each local model. If $\eta=0$, \eqref{mainP:FMTL} turns to an individual learning problem where each client learns its local model $w_k$ based on its own local data without collaboration with server or other clients. If $\eta > 0$, minimizing the Laplacian regularization term encourages the models of the neighboring clients to be close to each other. The impacts of the existing relationship between the models of the clients on the performance of our proposed algorithms will be shown in the later  section of experiment.}
		
		\begin{remark}
			There are other methods of regularization to encourage the models of the neighboring clients to be close to each other, e.g., using $||w_k\!-\!w_\ell||$ instead of $||w_k\!-\!w_\ell||^2$ in \eqref{regularization} as Network Lasso does \cite{jung21SPL,jung19SPL,david15},
			or using $\tr(\widehat{W}\Omega \widehat{W}^T)$ instead of \eqref{regularization} as MOCHA does \cite{smith17NIPS},
			where $\widehat{W}:=[w_1,\dots,w_N]\in\RRR^{d\times N}$. 
			On the other hand, problem \eqref{mainP:FMTL} is a generalization of the problem in \cite{hanzely20arXiv} where several algorithms are developed for strongly convex objectives. Problem \eqref{mainP:FMTL} is also similar to the generalized total variation minimization problem \cite{yasmin21arXiv} which is solved by a primal-dual method for convex objectives. 
			\cite{paul17} has a convex version of problem \eqref{mainP:FMTL} which is solved by a decentralized algorithm using Alternating Direction Method of Multipliers (ADMM). 
			In \eqref{mainP:FMTL}, we present the FTML problem using the Laplacian regularization matrix $\LL$. Utilizing the special properties of $\LL$, we
			successfully design FMTL algorithms using SGD. Importantly, our algorithms can work (i) in 
			both centralized and decentralized communication schemes, and (ii) with both strongly convex and nonconvex objective functions. 
		\end{remark}

		\begin{assumption}[Smoothness]
			\label{assump:smooth}
			For each $k\in\NN$, $F_k$ is $\beta$-smooth, i.e., for any $w, w'\in\RRR^d$,
			\begin{align*}
				\|\nabla F_k(w)-\nabla F_k(w')\|\leq \beta \|w-w'\|.
			\end{align*}
		\end{assumption}

		\begin{assumption}[Strong convexity]
			\label{assump:strongconvex}
			For each $k\in\NN$, $F_k$ is $\alpha$-strongly convex, i.e., for any $w, w'\in\RRR^d$,
			\begin{align*}
				F_k(w)\geq F_k(w')+\left\langle\nabla F_k(w'),w-w'\right\rangle
				+\frac{\alpha}{2}\|w-w'\|^2.
			\end{align*}
		\end{assumption}
		
		\begin{assumption}[Bounded variance]
			\label{assump:boundvariance}
			The set of $\nabla \widetilde{F}_k(w,\zeta_k)$, $k\in\NN$ is unbiased stochastic gradients of $\nabla F_k(w)$, $k\in\NN$, with total variance bounded by $\sigma_1^2$, i.e., for any $W\in\RRR^{dN}$,
			\begin{align*}
				\sum\nolimits_{k=1}^N\EEE_{\zeta_k}\|\nabla\widetilde{F}_k(w_k,\zeta_k)\!-\!\nabla F_k(w_k)\|^2\leq \sigma_1^2.
			\end{align*}
		\end{assumption}
		\vspace{-0mm}
		We note that Assumption~\ref{assump:boundvariance} is weaker than the assumption of individual bounded variance that is used at each client in FL and personalized FL problems \cite{sai20ICML,canh20NIPS,fallah20arXiv}.
		It should also be noted that \eqref{mainP:FMTL} shares some similarities to the multi-task learning problem of \cite{nassif20OJSP,nsassif20partII}. However, the latter requires that each $F_k(w_k)$ is twice differential with the Hessian $\nabla_{w_k}^2 F_k(w_k)$ uniformly bounded from below and above, which is more restrictive than our assumptions. Moreover, this problem does not take into account the issue of non-IID data distributions among clients, and thus it is not formulated for FL settings.

		\begin{figure}[t!]
			\centering
			\includegraphics[width=0.47\textwidth]{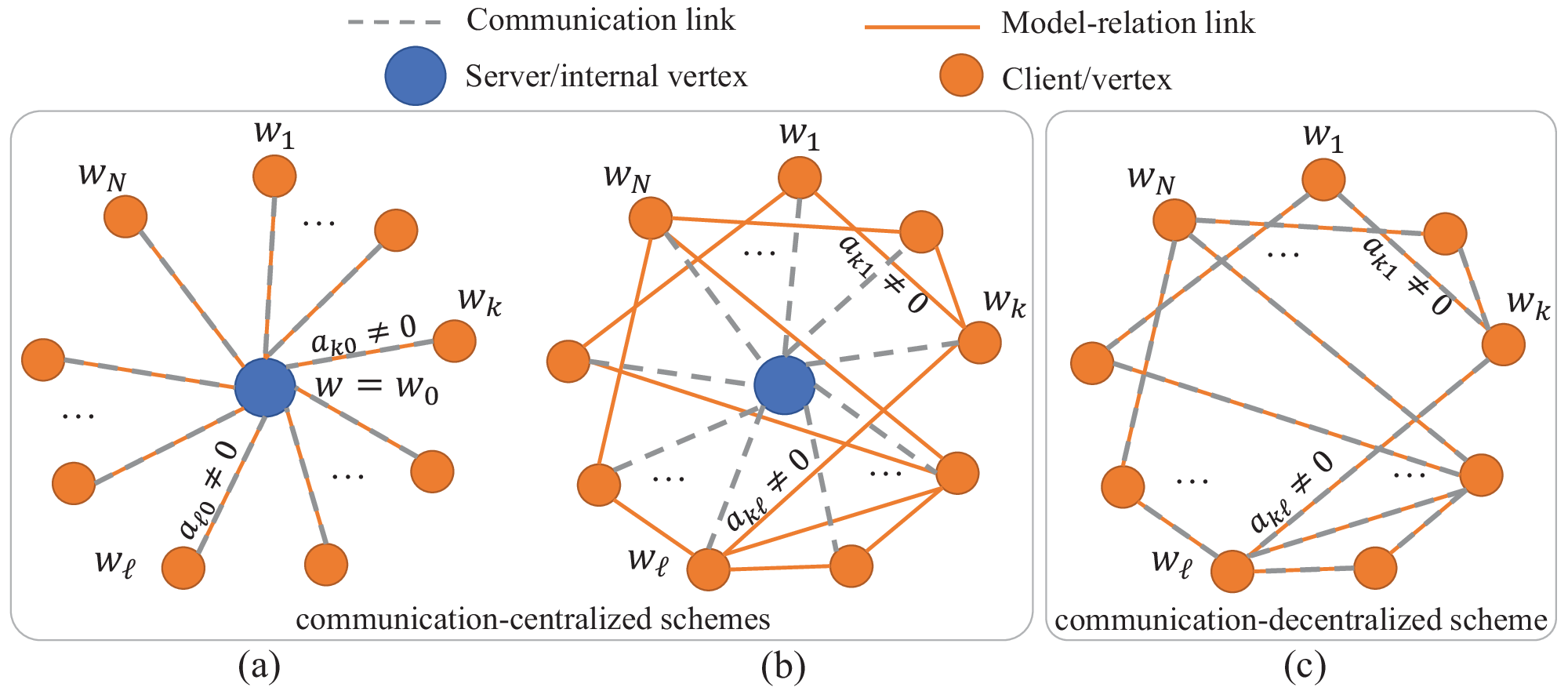}
			\vspace{-0mm}
			\caption{Illustrations of undirected weighted graphs in FL. (a): Star graph with a server for traditional FL and personalized FL; (b) and (c): Entity graph with and without server for FMTL}
			\label{fig:graph}
			\vspace{-5mm}
		\end{figure}
		
		\vspace{-0mm}
		\subsection{A New View of the FMTL Problem}
		\vspace{-0mm}
		
		\label{sec:unifiedproblem}
		We first observe that in conventional FL and personalized FL, all clients connect to a server under
		a communication-centralized scheme shown in Figure~\ref{fig:graph}(a). 
		The relationships among the models of the clients and the server are 
		presented by a star graph. 
		In this graph, a server is considered as a virtually internal vertex $0$ with its loss function $F_0=0$ and a model $w_0$. Here, all the models of clients are only related to the server model $w_0$, i.e., $a_{k0}> 0,\forall k$, but not with each other,  i.e., $a_{k\ell}=0,\forall k,\ell \neq 0$. In this work, we assume that the weights $a_{k\ell}$ are known and focus on the development of FMTL algorithms to solve problem \eqref{mainP:FMTL}. The finding of $a_{k\ell}$ in specific learning applications are referred to \cite{tuck19arXiv,tuck20arXiv}. In what follows, we show that the formulated FMTL problem \eqref{mainP:FMTL} can be used for the conventional FL and some types of personalized FL. For a more general optimization problem of personalized FL, we refer to LSGD-PFL \cite{hanzely21arXiv}.
		

		\textbf{Relation of FMTL to conventional FL:}
		The objective function of \eqref{mainP:FMTL} can be seen as a Lagrangian function of the following problem
		\begin{align}
			\min_{W} \sum\nolimits_{k=1}^N F_k(w_k) ,\,\text{s.t.}\,\, w_1 \!= w_2= \!\dots\! =\! w_N,
		\end{align}
		which is equivalent to the conventional FL problem (FedAvg) \cite{mcmahan17}. Therefore, the solution of the conventional FL problem can be obtained by solving \eqref{mainP:FMTL}.
		
		
		\textbf{Relation of FMTL to personalized FL with Moreau envelopes (pFedMe):} The problem of pFedMe \cite{canh20NIPS} is formulated as 
		\begin{align}
			\label{mainP:pFedMe}
			\underset{w}{\min} \,\,
			J(w)=\sum\nolimits_{k=1}^N\tilde{J}_k(w),
		\end{align}
		where $\tilde{J}_k(w)=\underset{z_k}{\min} \,\,  F_k(z_k)
		+ \frac{\eta}{2}||z_k-w||^2$. We observe that
		\begin{align}
			\nonumber
			J(w)
			&=\sum\nolimits_{k=1}^N \min_{z_k} \left(F_k(z_k) +\frac{\eta}{2}\|z_k-w\|^2\right)
			\\
			&=\min_{z_1,\dots,z_N} \sum\nolimits_{k=1}^N \left(F_k(z_k) +\frac{\eta}{2}\|z_k-w\|^2\right).\nonumber
		\end{align}
		Therefore, \eqref{mainP:pFedMe} is equivalent to the following problem with $z_0 =w$ and $F_0 \equiv 0$:
		\begin{equation*}
			\min_{z_0,z_1,\dots,z_N} \sum\nolimits_{k=0}^{N} F_k(z_k) +\frac{\eta}{2}\sum\nolimits_{k=0}^N \|z_k-z_0\|^2,
		\end{equation*}
		which is a special case of \eqref{mainP:FMTL} with the star graph topology and $a_{k0}=1,\forall k\in\NN$.
		
		\textbf{Relation of FMTL to meta-learning-based personalized FL (Per-FedAvg):} The problem of Per-FedAvg \cite{fallah20arXiv} is given by
		\begin{align}
			\label{mainP:PerFedAvg}
			\underset{w}{\min} \,\,
			& J(w)=\sum\nolimits_{k=1}^N\!F_k(w-\mu\nabla F_k(w)),
		\end{align}
		where $\mu>0$ and each $F_k$ is assumed to be $L_k$-Lipschitz continuous.
		Set $w_k =w-\mu\nabla F_k(w)$ and $\ell_k =\frac{L_k}{2}$, $k\in\NN$. Using Lemma~1.2.3 in \cite{nesterov18book} twice, we have that, for $\mu <\min_{k} \ell_k$ and for all $z_k\in \RRR^d$,
		\begin{align*}
			F_k(w_k) &\leq F_k(w) +\langle\nabla F_k(w),w_k-w\rangle +\ell_k\|w_k-w\|^2 \\
			&= F_k(w) -(\mu-\ell_k\mu^2)\|\nabla F_k(w)\|^2 \\
			&\leq F_k(z_k) +\langle\nabla F_k(w),z_k-w\rangle +\ell_k\|z_k-w\|^2 \\
			&\quad -(\mu-\ell_k\mu^2)\|\nabla F_k(w)\|^2 \\
			&= F_k(z_k) +a_{k0}\|z_k-w\|^2 \\
			&\quad -(\mu-\ell_k\mu^2)\left\|\nabla F_k(w) -\frac{z_k-w}{2(\mu-\ell_k\mu^2)}\right\|^2,
		\end{align*}
		where $a_{k0} :=\ell_k+\frac{1}{4(\mu-\ell_k\mu^2)}$. Hence,
		\begin{align*}
			F_k(w_k) \leq \min_{z_k}\left(F_k(z_k) +a_{k0}\|z_k-w\|^2\right),
		\end{align*}
		which implies that
		\begin{align*}
			J(w)
			&\leq \sum\nolimits_{k=1}^N
			\min_{z_k}\left(F_k(z_k) +a_{k0}\|z_k-w\|^2\right) \\
			&= \min_{z_1, \dots, z_N} \sum\nolimits_{k=1}^N \left(F_k(z_k) +a_{k0}\|z_k-w\|^2\right).
		\end{align*}
		Now, \eqref{mainP:PerFedAvg} can be solved through its following epigraph problem with $z_0 = w$ and $F_0 = 0$:
		\begin{equation*}
			\min_{z_0,z_1,\dots,z_N} \sum\nolimits_{k=0}^{N} F_k(z_k)
			+\frac{\eta}{2}
			\sum\nolimits_{k=0}^N a_{k0}\|z_k-z_0\|^2,
		\end{equation*}
		which is also a special case of \eqref{mainP:FMTL} with the star graph topology and $a_{k0}=1,\forall k\in\NN$.

		\section{Federated Multi-Task Learning: Algorithms}
		\vspace{-0mm}
		\label{sec:alg}
		\begin{algorithm}[!t]
			\caption{\OurAlg }
			\begin{algorithmic}[1]
				\label{alg:2SSGD}
				\STATE \textbf{client $k$'s input}: local step-size $\mu$
				\STATE \textbf{server's input}: graph information $\{a_{k\ell}\}$, initial $w_{k}^{(0)}, \forall k\in\NN$, and global step-size $\mut=\mu R$
				\FOR{each round $t = 0,\dots,T-1$}
				\STATE server uniformly samples a subset of clients $\SSS^{(t)}$ of size $S$ and sends $w_k^{(t)}$ to client $k,\forall k\in\SSS^{(t)}$
				\renewcommand{\algorithmicdo}{\textbf{in parallel do}}
				\WHILE{$k\in\SSS^{(t)}$}
				\renewcommand{\algorithmicdo}{\textbf{do}}
				\STATE initialize local model $w_{k,0}^{(t)}\leftarrow w_{k}^{(t)}$
				\FOR{$r=0,\dots,R-1$}
				\STATE compute mini-batch gradient $\nabla \widetilde{F}_k(w_{k,r}^{(t)})$
				\STATE $w_{k,r+1}^{(t)}\leftarrow w_{k,r}^{(t)} - \mu \nabla \widetilde{F}_k(w_{k,r}^{(t)})$
				\ENDFOR
				\STATE send $w_{k,R}^{(t)}$ to the server
				\ENDWHILE
				\renewcommand{\algorithmicwhile}{\textbf{on server}}
				\renewcommand{\algorithmicendwhile}{\textbf{end on server}}
				\renewcommand{\algorithmicdo}{\textbf{do}}
				\WHILE{$\!\!$}
				\STATE $w_{k,R}^{(t)} \leftarrow  w_{k}^{(t)}$, $\forall k\notin\SSS^{(t)}$
				\STATE $w_{k}^{(t+1)}\!\leftarrow\!w_{k,R}^{(t)}-\mut\eta\sum\nolimits_{\ell\in\NN_k
				} a_{k\ell}(w_{k,R}^{(t)}\!-\!w_{\ell,R}^{(t)})$,  $\forall k\in\SSS^{(t)}$
				\STATE $w_{k}^{(t+1)} \leftarrow  w_{k}^{(t)}$, $\forall k\notin\SSS^{(t)}$
				\ENDWHILE
				\ENDFOR
			\end{algorithmic}
		\end{algorithm}

		\subsection{\OurAlg: Communication-Centralized Algorithm}
		\vspace{-0mm}
		In this section, we propose an algorithm \OurAlg, which is presented in Algorithm~\ref{alg:2SSGD}, to solve the formulated FL problem \eqref{mainP:FMTL} under
		the
		communication-centralized scheme. Here, we use an entity graph to capture the relationships among the models of clients as shown in Figure~\ref{fig:graph}(b).\footnote{In an entity graph, each vertex is a value of an entity (e.g., a person) and an edge (e.g., friendship) between two entities exists if these entities are perceived to be similar \cite{tuck19arXiv}.}
		First, the server uniformly samples a subset of clients $\SSS^{(t)}$ and sends the latest update of local model $w_k$ to each client $k, \forall k\in\SSS^{(t)}$. Then, after $R$ local update steps are performed, the server receives the latest local update from the sampled clients to perform model regularization
		for each local model.
		
		Note that in the entity graph, the models of clients are only related to other models but not to
		any server model, 
		as in the star graph of the conventional FL and personalized FL.
		Therefore, \OurAlg has a key difference compared to the conventional FL algorithms (e.g., FedAvg \cite{mcmahan17}) and the personalized FL algorithms (e.g., pFedMe \cite{canh20NIPS}, and Per-FedAvg \cite{fallah20arXiv}). Instead of updating the personalized models only at the clients using a global model from the server, \OurAlg directly updates each local model at both client and server sides without building a global model.

		Specifically, as shown in Figure~\ref{fig:updatesteps}, in each communication round, each client $k\in\SSS^{(t)}$ copies its current local model received from the server: $w_{k,0}^{(t)}=w_k^{(t)}$, and perform $R$ local updates of the form:
		\begin{align*}
			w_{k,r+1}^{(t)}\leftarrow w_{k,r}^{(t)} - \mu \nabla \widetilde{F}_k(w_{k,r}^{(t)}),
		\end{align*}
		where $\mu$ is the local step-size.
		Then server receives $\{w_{k,R}^{(t)}\}$ from sampled clients $k\in\SSS^{(t)}$, and updates
		\begin{align*}
			w_{k,R}^{(t)} \leftarrow  w_{k}^{(t)},
		\end{align*}
		for any non-sampled client $k\notin\SSS^{(t)}$. Finally, the server performs its regularization update for any sampled client $k\in\SSS^{(t)}$ as
		\begin{align*}
			\!\! w_{k}^{(t+1)}&\leftarrow w_{k,R}^{(t)}-\mut\eta\sum\nolimits_{\ell\in\NN_k\cap \SSS^{(t)}}
			a_{k\ell}(w_{k,R}^{(t)}-w_{\ell,R}^{(t)}),
		\end{align*}
		and for any non-sampled client $k\notin\SSS^{(t)}$ as
		\begin{align*}
			w_{k}^{(t+1)} \leftarrow  w_{k}^{(t)},
		\end{align*}
		where $\mut=\mu R$ is a global step-size. This step finishes one round of communication.

		{The mechanism of \OurAlg is explained with $N=2$ example clients as seen in Figure~\ref{fig:updatesteps}. The two clients are the neighbors of each other and share a certain similarity model. Let $(w_1^*, w_2^*)$ be the global solution (true optimum or true opt.) to problem \eqref{mainP:FMTL}, which is presented by orange squares. Denote by $(\widehat{w}_1^*, \widehat{w}_2^*)$ be the local solution (client optimum or client opt.) that obtains the minimum of the local lost function $F_k(w_k)$, which is presented by blue squares. In the case of non i.i.d data, $\widehat{w}_1^*$ and $\widehat{w}_2^*$ are far away from each other, and $(\widehat{w}_1^*, \widehat{w}_2^*)$ is also far away from $(w_1^*, w_2^*)$. At round $t$, after making $R=3$ local updates, the updated models $(w_{1,R}^{(t)}, w_{2,R}^{(t)})$ (blue circles) are moved closer to $(\widehat{w}_1^*, \widehat{w}_2^*)$. Then, we make a further step of regularization update in order to move $w_{1,R}^{(t)}$ toward $\widehat{w}_2^*$ and also move $w_{2,R}^{(t)}$ toward $\widehat{w}_1^*$, which finally makes the updated model after round $t$, i.e., $(w_{1}^{(t+1)}, w_{2}^{(t+1)})$, closer to $(w_1^*, w_2^*)$. By doing local and regularization updates in each round, the converged solution of \OurAlg will be $(w_1^*, w_2^*)$.}
			
		\begin{figure}[!t]
			\centering
			\includegraphics[width=0.47\textwidth]{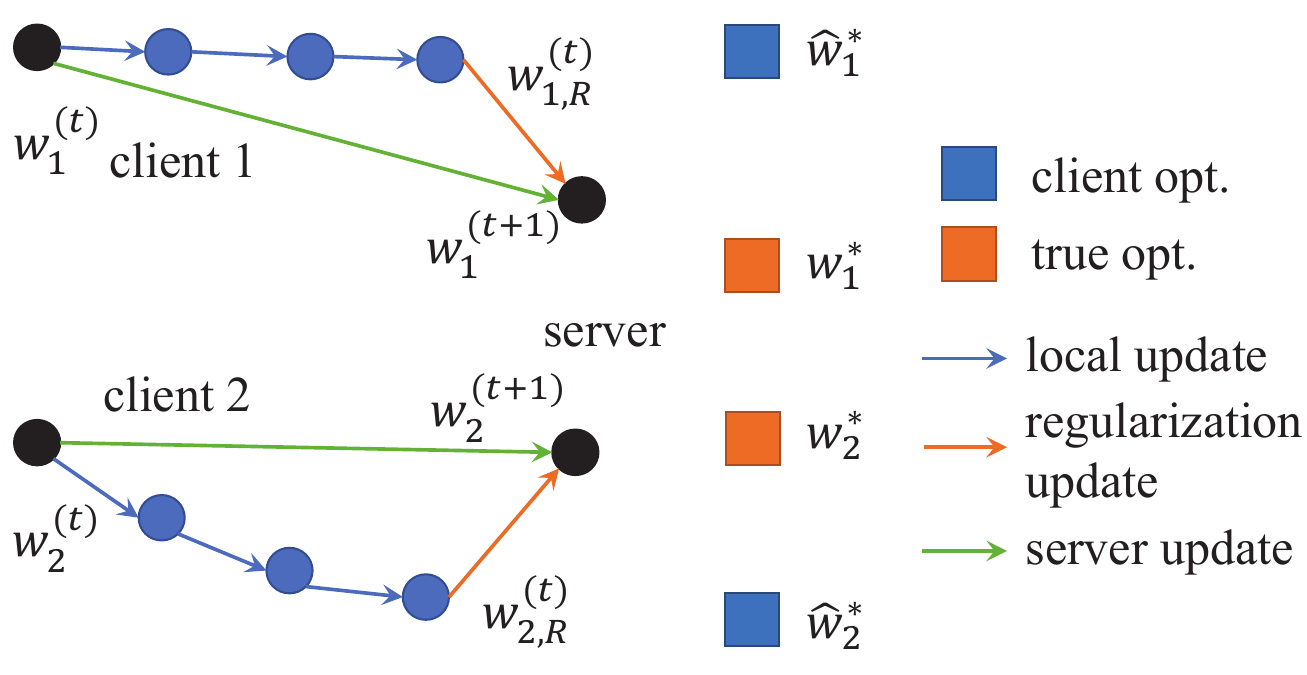}
			\vspace{-0mm}
			\caption{
				{The update
		steps of \OurAlg at both client and server sides are illustrated for $2$ related tasks (clients) with $3$ local steps $(N=2,R=3)$ at round $t$. The local updates $w_{k,r}^{(t)}$ (blue circles) move towards the client optima $\widehat{w}_{k}^*$ (blue square). The regularization updates (in orange) ensures the server update (in green) moves towards the true optimum $w_{k}^*,\forall k\in\NN$ (orange square).} }
			\label{fig:updatesteps}
			\vspace{-0mm}
		\end{figure}
		
		\vspace{-0mm}
		\subsection{\dFed: Decentralized Version of \OurAlg}
		\vspace{-0mm}
		We note that the server in \OurAlg needs to known all the graph information $\{a_{k\ell}\}$. This requirement can be achieved by letting all the clients 
		send the information of their neighbors to the server at the beginning of the learning process. However, in a network of massive clients (e.g., thousands), it might be impractical to maintain all the information 
		of the graph (e.g., vertices, weighted edge) as well as storage for all model updates at the server. This motivates us to propose $\dFed$, which is a decentralized version of \OurAlg, and 
		presented in Algorithm~\ref{alg:dFed}.

		Specifically, in each communication round,
		each client of an entity graph (as shown in Figure~\ref{fig:graph}(c))
		performs $R$ local updates, and sends its updated model to their neighboring clients to perform the model regularization. {Here, each client does not need to communicate with the rest of the large number of clients in the whole network. Each client only needs to communicate with its neighbor clients. A client $\ell$ is a neighbor of client $k$ if and only if it has a communication link (i.e., $a_{k\ell} \neq 0$) and share a certain model similarity with client $k$ (i.e., $a_{k\ell} > 0$). The set of neighboring clients of client $k$ is defined as $\widetilde{\NN}_k = \{\ell \,\,|\,\, a_{k\ell} > 0\}$. 
		}
	    Note that because there is no server 
		for
		coordinating the learning, there is no client sampling in \dFed. Compared to the non-FL decentralized scheme \cite{nassif20OJSP,nsassif20partII}, \dFed uses $R$ local updates, which are typical in FL algorithm designs.

		\vspace{-0mm}
		\section{Federated Multi-Task Learning:\\ Convergence Rate}
		\vspace{-0mm}
		\label{rateofconv}
		In this section, we present the convergence rate of \OurAlg and \dFed.
		Let $W^*=[w_1^*,\dots,w_N^*]$ be the optimal solution to \eqref{mainP:FMTL}.
		\vspace{-0pt}
		\begin{lemma}
			\label{lemma:Boundedgradient}
			Suppose that Assumption~\ref{assump:smooth} holds and $\eta\rho > 2\beta$, where $\rho := \|\LL\|$. Then there exists $\sigma_2\geq 0$, e.g., $\sigma_2=\|\nabla F(0)\|\sqrt{\frac{\eta\rho}{\eta\rho-2\beta}}$ such that, for any $W\in\RRR^{dN}$,
			\begin{align}\label{sigma3main}
				\sum\nolimits_{k=1}^N \|\nabla F_k(w_k)\|^2 \!\leq\! \sigma_2^2 \!+\! \sum\nolimits_{k=1}^N\|\nabla_{w_k} J(W)\|^2,
			\end{align}
			where $\nabla_{w_k} J(W)$ is the gradient of $J$ with respect to $w_k$.
			Consequently, if every $F_k$ is convex, then
			\begin{align}\label{sigma2main}
				\sum\nolimits_{k=1}^N \|\nabla F_k(w_k^*)\|^2 \!\leq\! \sigma_2^2.
			\end{align}
		\end{lemma}
		\begin{proof}
			See Appendix B. 
		\end{proof}
		
		For any given value of $\rho$, the condition $\eta\rho > 2\beta$ in Lemma~\ref{lemma:Boundedgradient} can be always achieved by tuning $\eta\in\RRR$.
		Therefore, the impact of the relationships among the models of clients (or the graph Laplacian structure encoded by $\rho$) on the convergence of \OurAlg and \dFed can be controlled by $\eta$. One can choose a large $\eta$ if $\rho$ is small and vice versa to satisfy this condition.
		
		Note that in the conventional FL setting, i.e., $w_k=w,\forall k\in\NN$, \eqref{sigma3main} is rewritten as
		\begin{align*}
			\!\frac{1}{N}\sum\nolimits_{k=1}^N \|\nabla F_k(w)\|^2 \!\leq\! \frac{\sigma_2^2}{N} \!+\! \gamma^2\|\nabla_w J(W)\|^2\,\,\text{with}\,\,\gamma\!=\! 1,
		\end{align*}
		which is exactly the assumptions of $(\sigma_2/\sqrt{N},\gamma)$-bounded gradient dissimilarity in \cite{sai20ICML,khaled20AIS}, and the $\gamma$-local dissimilarity in \cite{li20} with $\sigma_2=0$.
		Here, $\sigma_2 = 0$ and $\gamma= 1$ are for the i.i.d cases, while $\sigma_2 \geq 0$ and $\gamma\geq 1$ for non-IID cases.
		
		\begin{algorithm}[!t]
			\caption{\dFed  --Decentralized \OurAlg}
			\begin{algorithmic}[1]
				\label{alg:dFed}
							\STATE \textbf{client $k$'s input}: $\{a_{k\ell}\}$, {$\widetilde{\NN}_k$}, initial $w_{k}^{(0)},  \forall k\in\NN$, local step-size $\mu$, and global step-size $\mut=\mu R$
				\FOR{each round $t = 0,\dots,T-1$}
				\renewcommand{\algorithmicdo}{\textbf{in parallel do}}
				\WHILE{$k\in\NN$}
				\renewcommand{\algorithmicdo}{\textbf{do}}
				\STATE initialize local model $w_{k,0}^{(t)}\leftarrow w_{k}^{(t)}$
				\FOR{$r=0,\dots,R-1$}
				\STATE compute mini-batch gradient $\nabla \widetilde{F}_k(w_{k,r}^{(t)})$
				\STATE $w_{k,r+1}^{(t)}\leftarrow w_{k,r}^{(t)} - \mu \nabla \widetilde{F}_k(w_{k,r}^{(t)})$
				\ENDFOR
					\STATE send $w_{k,R}^{(t)}$ to {its neighboring clients in $\widetilde{\NN}_k$}
				\ENDWHILE
				\WHILE{$k\in\NN$}
						\STATE \!\!\!$w_{k}^{(t+1)}\leftarrow w_{k,R}^{(t)}-\mut\eta\sum\nolimits_{\ell\in{\widetilde{\NN}_k}}
			a_{k\ell}(w_{k,R}^{(t)}-w_{\ell,R}^{(t)})$
				\ENDWHILE
				\ENDFOR
			\end{algorithmic}
			\vspace{-0mm}
		\end{algorithm}

		From now on, let $\sigma_2$ and $\rho$ be defined as in Lemma~\ref{lemma:Boundedgradient}, and $W^{(t)}=[w_1^{(t)},\dots,w_N^{(t)}]$ be the collective vector generated by \OurAlg (with client sampling) or \dFed  (without client sampling, i.e., $S=N$) at round $t$. Node that 
		the convergence rate of \dFed is obtained directly from the convergence rate of \OurAlg when $S=N$. In the following theorems, we show that
		\OurAlg admits linear speedup for strongly convex and sublinear speedup of order $1/2$ for nonconvex objective functions.
		
		\begin{theorem}[Convergence in strongly convex cases]
			\label{theorem1}
			Suppose that Assumptions~\ref{assump:smooth},~\ref{assump:strongconvex}, and~\ref{assump:boundvariance} hold, and $\eta > \frac{2\beta}{\rho}$. Then there exists $\mu\leq \frac{\mut_1}{R}$ such that, for any $T\geq\frac{4N}{\mut_1\alpha S}$,
			\begin{align}
				\label{theorem1a}
				\nonumber
				&\!\EEE[J(\widetilde{W}^{(T)})-J(W^*)]
				\!\leq\! \widetilde{\OO}\bigg(\alpha\Delta^{(0)} e^{-\frac{\mut_1\alpha S T}{4N}}
				\!\!+\! \frac{\sigma_1^2}{(\alpha T)^2RS}
				\\
				&\qquad\qquad\qquad\quad\,\,
				+ \frac{\sigma_2^2}{(\alpha T)^2S}
				+ \frac{\sigma_1^2}{\alpha TRS}
				+ \frac{\sigma_2^2}{\alpha TS}
				\bigg),
			\end{align}
			where $\mut_1:=\min\left\{\frac{1}{q},\frac{2}{\eta\rho}\right\}$,
			$q=\frac{128\beta ^2\eta\rho}{\alpha^2}
			+ 12(\beta+\eta\rho)
			+ \frac{96\beta ^2}{\alpha}
			+ \frac{32p\beta ^2}{\alpha\eta\rho}$,
			$p=2(\beta+\eta\rho)
			+ \frac{8\eta^2\rho^2}{\alpha}
			+ \frac{64\beta ^2}{\alpha}
			+ \frac{12(\beta+\eta\rho)^2}{\eta \rho}
			+ 6\eta\rho
			+ \frac{48\beta ^2}{\eta \rho}$,
			$\Delta^{(0)}:= \|W^{(0)}-W^*\|^2$, $\widetilde{W}^{(T)}:=\sum\nolimits_{t=0}^{T-1}\frac{\theta^{(t)}W^{(t)}}{\Theta_T}$, $\Theta_T = \sum\nolimits_{t=0}^{T-1}\theta^{(t)}$, $\theta^{(t)}=\left(1-\mu R S\alpha/(4N)\right)^{-(t+1)}$, and $\widetilde{\OO}$ hides both constants and polylogarithmic factors. Consequently, the output of \OurAlg has expected error smaller than $\varepsilon$ when 
			\begin{align}
				\label{theorem1b}
				T \!=\! \widetilde{\OO}\!\left(\frac{1}{\alpha S}
				\!\!+\! \frac{\sigma_1}{\alpha \sqrt{\varepsilon RS}}
				\!\!+\! \frac{\sigma_2}{\alpha \sqrt{\varepsilon S}}
				\!\!+\! \frac{\sigma_1^2}{\alpha RS \varepsilon}
				\!\!+\! \frac{\sigma_2^2}{\alpha S \varepsilon}\right).
			\end{align}
		\end{theorem}
		\begin{proof}
			See Appendix D. 
		\end{proof}
		
		\begin{theorem}[Convergence in nonconvex cases]
			\label{theorem2}
			Suppose that Assumptions~\ref{assump:smooth} and ~\ref{assump:boundvariance} hold, and $\eta > \frac{2\beta}{\rho}$. Then there exists $\mu\leq\frac{\mut_2}{R}$ such that, for any $T >0$,
			\begin{align}
				\label{theorem2a}
				\!\!\!\!\EEE\|\nabla J(W^{(t^*)})\|^2
				\!\leq\! \OO\!\left(\frac{\Delta_J}{TS}
				\!+\! \frac{\Delta_J^{\frac{2}{3}} M^{\frac{2}{3}}}{T^{\frac{2}{3}}(RS)^{\frac{1}{3}}}
				\!+\! \frac{\Delta_J^{\frac{1}{2}}M^2}{\sqrt{TRS}}\right)\!,
			\end{align}
			where $\mut_2\!:=\!\min\left\{\frac{1}{v},\frac{2}{\eta\rho}\right\}$,
			$v=8\big(8\eta\rho\!+\!3(\beta+\eta\rho)\!+\!12(\beta+\eta\rho)\!+\!\frac{8u}{\eta\rho}\big)$,
			$u\!=\!\frac{(\beta+\eta\rho)^2}{2} + 2\eta^2\rho^2 + 16\eta\rho\beta ^2
			\!+\! \frac{6(\beta+\eta\rho)^3}{\eta\rho} \!+\! 3\eta\rho(\beta+\eta\rho) \!+\! \frac{24(\beta+\eta\rho)\beta ^2}{\eta\rho}$; $\Delta_J := J(W^{(0)})-J(W^*)$, $M^2 = R\sigma_2^2+\sigma_1^2$, and $t^*$ uniformly sampled from $\{0,\dots,T-1\}$.
			Consequently, the output of \OurAlg has expected error smaller than $\varepsilon$ when 
			\begin{align}
				\label{theorem2b}
				T \!=\! \OO\left(\frac{1}{S\varepsilon}
				\!\!+\! \frac{\sigma_1}{\varepsilon^{\frac{3}{2}}\sqrt{RS}}
				\!\!+\! \frac{\sigma_2}{\varepsilon^{\frac{3}{2}}\sqrt{S}}
				\!\!+\! \frac{\sigma_1^2}{\varepsilon^2 RS}
				\!\!+\! \frac{\sigma_2^2}{\varepsilon^2 S}\right).
			\end{align}
		\end{theorem}
		\begin{proof}
			See Appendix E. 
		\end{proof}
		
		For illustrative purposes, we compare our rates with those of FL and personalized FL algorithms in i.i.d cases (i.e., $\sigma_2=0$ and $\gamma=1$). The strongly-convex rate of \OurAlg becomes
		$\frac{\sigma_1^2}{\alpha RS\varepsilon}+\frac{1}{\alpha S}$, which matches the lower-bound for the identical case \cite{wood19NIPS}, compared to the latest $\frac{\sigma_1^2}{\alpha RS\varepsilon}+\frac{1}{\alpha}$ by SCAFFOLD \cite{sai20ICML} and $\frac{\sigma_1^2}{\alpha RS\varepsilon}+\frac{\delta}{\alpha}$ by LSGD-PFL \cite{hanzely21arXiv} with $\delta \geq 0$. Our rate improvement comes from the advantage of additional information about the structure of the models of clients that is captured by Laplacian regularization.
		Also, when no variance ($\sigma_1^2=0$) and no client sampling, the nonconvex rate of \OurAlg is $\frac{\sigma_2^2}{\varepsilon^2S}\!+\! \frac{\sigma_2}{\varepsilon^{3/2}} \!+\!\frac{1}{\varepsilon}$, which is tighter (without $\gamma$) than the rate of SCAFFOLD, and less dependent on $\sigma_2$ than that of \cite{yu19AAAI}.

		
		\input{experiments_ton}
		\vspace{-0mm}
		\section{Conclusion}
		\vspace{-0mm}
		This work has formulated a FMTL problem
		using Laplacian regularization to capture the relationships among the models of clients. 
		The formulated problem has been proved to be used for traditional FL and personalized FL. We have also proposed both communication-centralized and decentralized algorithms to solve the formulated problem with guaranteed convergence to the optimal solution. Theoretical results show that our algorithms \OurAlg and \dFed achieve the state-of-the art convergence rates. Experimental results with real datasets in both convex and nonconvex objectives demonstrate that the proposed algorithms outperform the conventional MOCHA in FMTL settings, the vanilla FedAvg in FL settings, and pFedMe, and Per-FedAvg in personalized FL settings.
		
		\bibliographystyle{IEEEtran}
		\bibliography{references,Canh_bib}
		\newpage
		\onecolumn
		\appendix
		\input{analysis}

\input{experiments_appendix}

		
	\end{document}

%% file: introduction.tex
\section{Introduction}
\vspace{-0mm}
{Recently, federated learning (FL) has been considered as a promising distributed and privacy-preserving method for building a global model from a massive number of hand-held devices \cite{mcmahan17,kairouz21,Felix20,Felix21}. FL has a wide range of futuristic applications, such as detecting the symptoms of possible diseases (e.g., stroke, heart attack, diabetes) from wearable devices in health-care systems \cite{rieke20,xu20,brisimi18}, or predicting disaster risks from internet-of-things devices in smart cities \cite{jiang20,ahmed20}.
In FL, one of the key challenges is the naturally non-IID data distributions among clients \cite{sai20ICML,haddadpour19}. When the differences among clients' data distributions increase,
the generalization error of the FL global model on each client's local data significantly increases \cite{li19,deng20}.}

{Personalized FL \cite{canh20NIPS,fallah20arXiv} and federated multi-task learning (FMTL) \cite{smith17NIPS} have been proposed as  solutions to handle non-IID data distributions among clients. Personalized FL aims to build a global model that is
leveraged to find a ``personalized model'' for each client's local data.
Here, the global model is considered as an ``agreed point'' for each client to start personalizing its model based on its heterogeneous local data distribution. }
Different from personalized FL, FMTL aims to simultaneously learn separate models, which is motivated by multi-task learning frameworks \cite{kumar12,zhang10}. Each of these models fits the data distribution of  each client. Therefore, FMTL directly addresses the issue stemming from non-IID data distributions without building any global model as personalized FL.

{On the other hand, from the aspect of the local data at clients, it is observed that the clients with similar features (e.g, location, time, age, gender) are likely to share similar behaviors. Therefore, although the clients' models are separated, they are normally 
related to each other.}
In FMTL, the relationships among the clients' models are captured by a regularization term which is minimized to encourage the  clients' models to be mutually impacted.
Unfortunately, these relationships have not been clearly taken into consideration
in the FMTL problem. Moreover, communication-decentralized and non-convex FMTL algorithms with guaranteed convergence are generally less explored. 

The main contributions of this work are as follows:
\vspace{-0mm}
\begin{itemize}[noitemsep,nolistsep]
	\item We formulate a FMTL problem using Laplacian regularization to explicitly leverage the relationships among the models of clients. 
	We then  introduce a new view
	of the FMTL problem
	that the formulated FMTL problem can be used not only for the conventional FL but also personalized FL.
	\item We propose a communication-centralized FMTL algorithm \OurAlg, and its decentralized version \dFed to solve the formulated FMTL problem.  We also  analyze the convergence rate of FMTL algorithms with both convex and nonconvex objective functions. In particular, \OurAlg and \dFed are proved to achieve a linear speedup (resp. sublinear speedup of order $1/2$) for strongly convex (resp. nonconvex) objective cases.
	\item 
	We empirically evaluate the performance of \OurAlg and \dFed using real datasets that capture the non-IID data distribution among clients. We show that in terms of local accuracy, \OurAlg and \dFed outperform the  traditional algorithm FedAvg in FL settings, the conventional algorithm MOCHA in FMTL settings, as well as pFedMe and Per-FedAvg in personalized FL settings.
\end{itemize}

\input{related_work}

%% file: related_work.tex
\vspace{-0mm}
\section{Related Work}
\vspace{-0mm}

\textbf{Federated Learning.}
One of the earliest work of FL is FedAvg \cite{mcmahan17}, which builds the global model based on averaging
the local Stochastic Gradient Descent (SGD) updates. Various methods \cite{li20, zhao_federated_2018, haddadpour19, li_convergence_2020, khaled20AIS} are introduced to improve the robustness of the global model under non-i.i.d settings. For example, FedProx \cite{li20} adds a proximal term to the local objective, therefore addressing the statistical heterogeneity of clients.

\textbf{Personalized Federated Learning.}
Several personalized FL approaches have been proposed to tackle the issues sterming form non-IID data in the conventional FL. Mixture methods \cite{hanzely_federated_2020,deng20}  attempted to combine a local model with the global model, while \cite{liangThinkLocallyAct2020} applied this mixing to jointly learns compact local representations on each client and a global model across all clients.  
Motivating by
creating a well-generalized global model to quickly adapt to client's data after few gradient descent steps, pFedMe \cite{canh20NIPS} used Moreau envelopes, while Per-FedAvg \cite{fallah20arXiv} took advances of meta learning approaches: model-agnostic meta-learning  \cite{finn17ICML}. \cite{jiang_improving_2019} proposed the combination of FedAvg and Reptile \cite{nichol_first-order_2018} to improve FL personalization. A different personalized FL approach to train deep neural networks is FedPer \cite{arivazhagan_federated_2019}. Clients share a set of base layers with a server and keep personalization layers that adapt quickly to the local data.

\textbf{Federated Multi-Task Learning.} 
Another approach to deal with the non-IID data distributions at clients is learning separate
models each of which fits
each local data distribution. 
In this sense, FMTL was first introduced in \cite{smith17NIPS} where a systems-aware optimization framework MOCHA for handling stragglers and fault tolerance in FL settings is proposed. 
Besides that, there are also several other works studying FMTL. 
\cite{yasmin21arXiv} proposed a framework for generalized total variation minimization, which is useful in FMTL networks.
\cite{li20arXiv} introduced a FMTL algorithm to deal with the issues of accuracy, fairness and robustness in FL.
By treating the FL network as a star-shaped Bayesian network, \cite{luca19arXiv} developed a FMTL algorithm using approximated variational inference.
\cite{rui19ICBD} focused on a FMTL algorithm for online applications.
However, in all these works, the convergence rate of FMTL with nonconvex objectives has not been studied. Moreover, the relations among the problems of FMTL, the standard FL, and personalized FL are not yet investigated in the literature.

%% file: experiments_ton.tex
\section{Experiments}
\label{simulation}

In this section, we evaluate the performance of \OurAlg when the data are heterogeneous and non-i.i.d in both strongly convex and nonconvex settings. We show vital show the advances of \OurAlg with Laplacian regularization in federated multi-task and personalized settings by comparing \OurAlg with cutting-edge learning algorithms including MOCHA, pFedMe, Per-FedAvg, {FedProx \cite{li20}, SCAFFOLD \cite{sai20ICML}, AFL \cite{mohri_agnostic_2019}}, and the vanilla FedAvg. The experimental results show that  \OurAlg achieves appreciable performance improvement over others in terms of test accuracy.

\subsection{Experimental Settings}
We consider classification problems using real datasets generated in federated settings, including  Human Activity Recognition, Vehicle sensor, MNIST, and CIFAR-10. 
\begin{itemize}[noitemsep,nolistsep]
	\item  \textbf{Human Activity Recognition}: The set of data gathered from accelerometers and gyroscopes of cell phones from 30 individuals performing six different activities including lying-down, standing, walking, sitting, walking-upstairs, and walking-downstairs \cite{anguitaPublicDomainDataset2013}. Each individual is considered as a task (client) classifying 6 different activities.
	\item  \textbf{Vehicle Sensor}: Data is collected from a distributed wireless sensor network of 23 sensors including acoustic (microphone), seismic (geophone), and infrared (polarized IR sensor) \cite{duarteVehicleClassificationDistributed2004}. It aims to classify types of moving vehicles. We consider each sensor as a separate task (client) performing the binary classification to predict 2 vehicle types: Assault Amphibian Vehicle (AAV)  and Dragon Wagon (DW).
	\item \textbf{MNIST}:  A handwritten digit dataset \cite{lecunGradientbasedLearningApplied1998} includes 10 labels and 70,000 instances. The whole dataset is distributed  to $N = 100$ clients. Each client has a different local data size and consists of 2 over 10 labels.
	\item \textbf{CIFAR-10}:  An object recognition dataset \cite{krizhevskyLearningMultipleLayers} includes 60,000 colour images belonging to 10 classes. We partition the dataset to $N = 20$ clients and 3 labels per client.
\end{itemize}

\vspace{-0mm}
In practical FL networks, some clients have significantly limited data sizes and need collaborative learning with others. For each dataset, we hence down-sample 80\% data belonged to a half of the total clients to observe behaviour of all algorithms. We provide all details about  datasets and  results without down-sampling in the Appendix F. All datasets are split randomly with 75\% and 25\% for training and testing, respectively. 

We use a multinomial logistic regression model (MLR) with cross-entropy loss functions and $L_2$-regularization term as the strongly convex model for Human Activity Recognition, Vehicle Sensor, and MNIST.  
For nonconvex setting, we use  a simple deep neural network (DNN) with one hidden layer, a ReLU activation function, and a softmax layer at the end of the network for Human Activity and Vehicle Sensor datasets. The size of hidden layer is 100 for Human Activity and 20 for Vehicle Sensor. In the case of MNIST, we use DNN with 2 hidden layers and both layers have the same size of 100.  For CIFAR-10, we follow the CNN structure of \cite{mcmahan17}. 

The structural dependence matrix $\Omega$ of MOCHA is chosen  as $\Omega = (\textbf{I}_{N \times N} - \frac{1}{N}\textbf{11}^T)^2$ following settings of \cite{smith17NIPS,liangThinkLocallyAct2020}, where $\textbf{I}_{N \times N}$ is the identity matrix with size $N \times N$ and $\textbf{1}$ is a vector of all ones size $N$. Here, $\Omega$ is exactly the Laplacian matrix $L$ in problem \eqref{mainP:FMTL} when all the weights $a_{k\ell}=1, \forall k,\ell$. 
As both \OurAlg and \dFed have the same performance when there is no client sampling, in our experiments, we only evaluate the performance of \OurAlg.  When comparing \OurAlg with other algorithms, we conduct 5-fold cross-validation to figure out the combination of hyperparameters allowing each algorithm to achieve the highest test accuracy.
All experiments are implemented using PyTorch \cite{paszkePyTorchImperativeStyle2019} version 1.6. We follow the implementations  of \cite{canh20NIPS} for pFedMe, FedAvg, and Per-FedAvg, \cite{liangThinkLocallyAct2020} for MOCHA. All  experiments are run on \textbf{NVIDIA Tesla T4} GPU. 
All code and data are published at \footnote{\url{https://github.com/dual-grp/FedU_FMTL}}. The accuracy is reported with mean and standard deviation over 10 runs. 
\subsection{Performance of \OurAlg  in Federated Multi-Task Learning}
\label{sim:MTL}
\begin{figure*}[ht]
	\centering
	\begin{subfigure}{1\textwidth} 
		\centering
		{\includegraphics[scale=0.355]{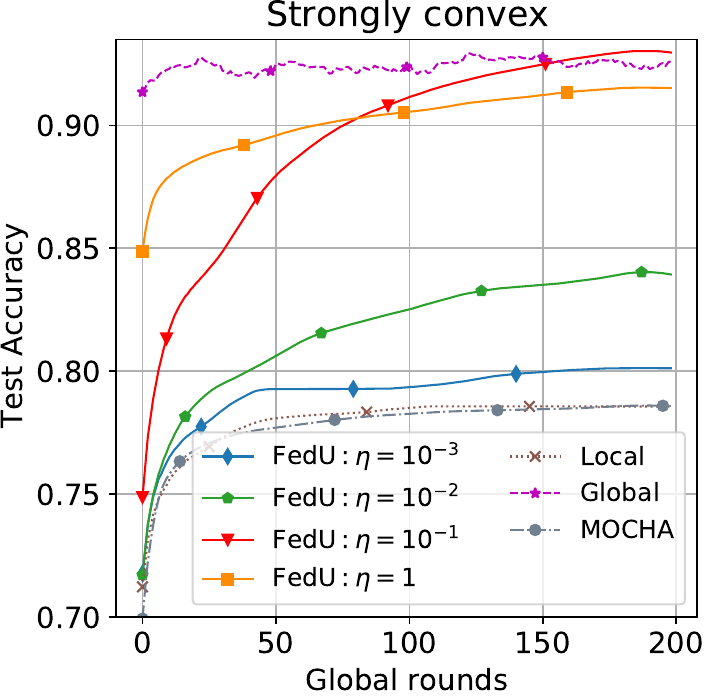}}\quad
		{\includegraphics[scale=0.355]{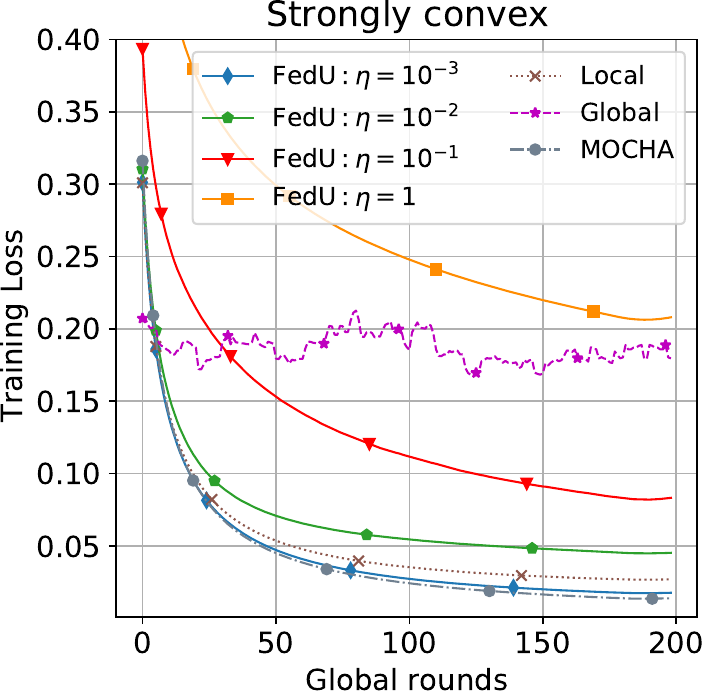}}\quad
		{\includegraphics[scale=0.355]{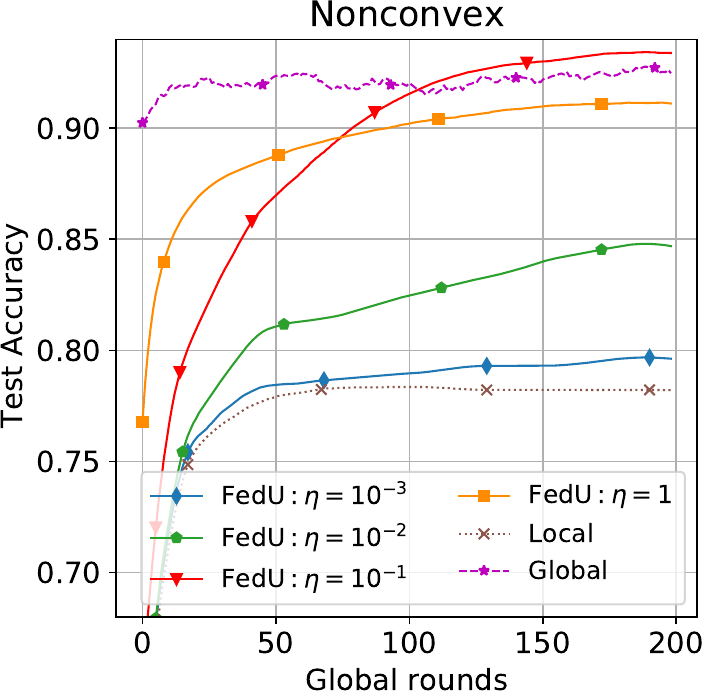}}\quad
		{\includegraphics[scale=0.355]{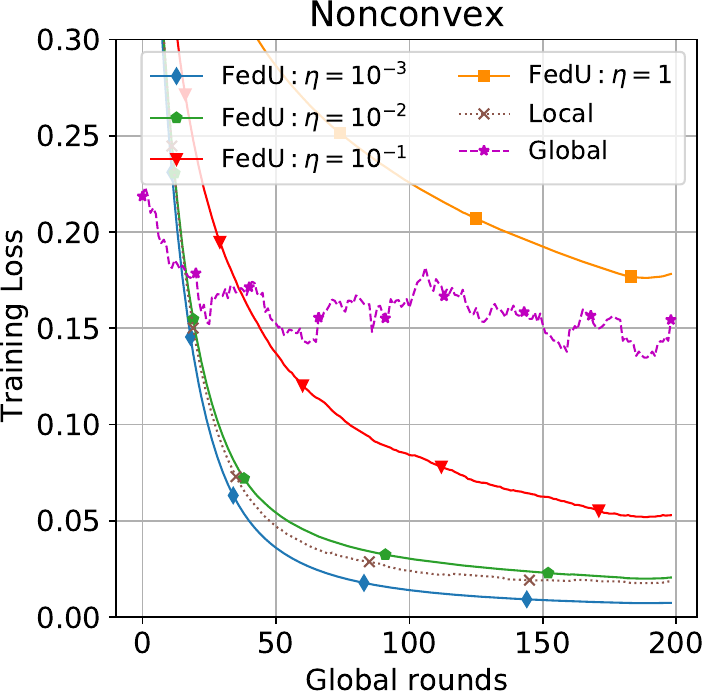}}
		\caption{Human Activity.}
		\label{F:Subeta_a}
	\end{subfigure}
	\begin{subfigure}{1\textwidth} 
		\centering
		{\includegraphics[scale=0.355]{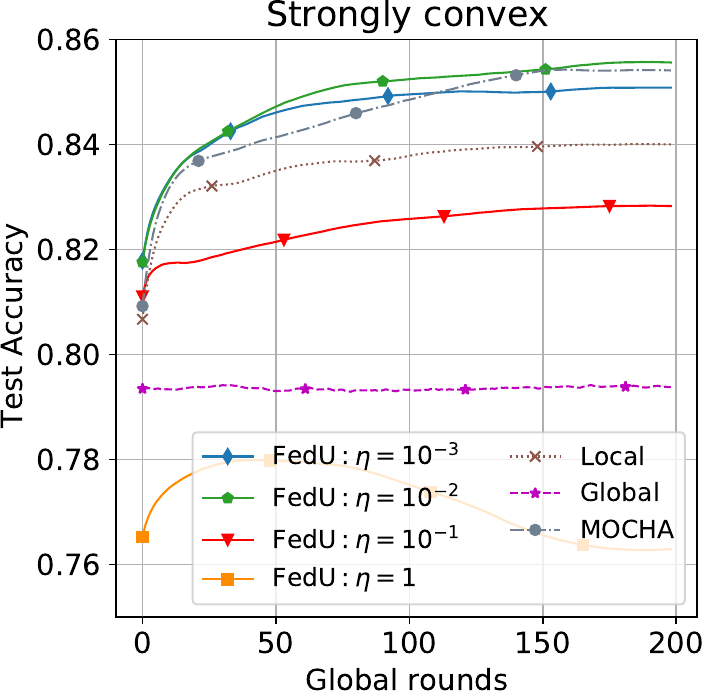}}\quad
		{\includegraphics[scale=0.355]{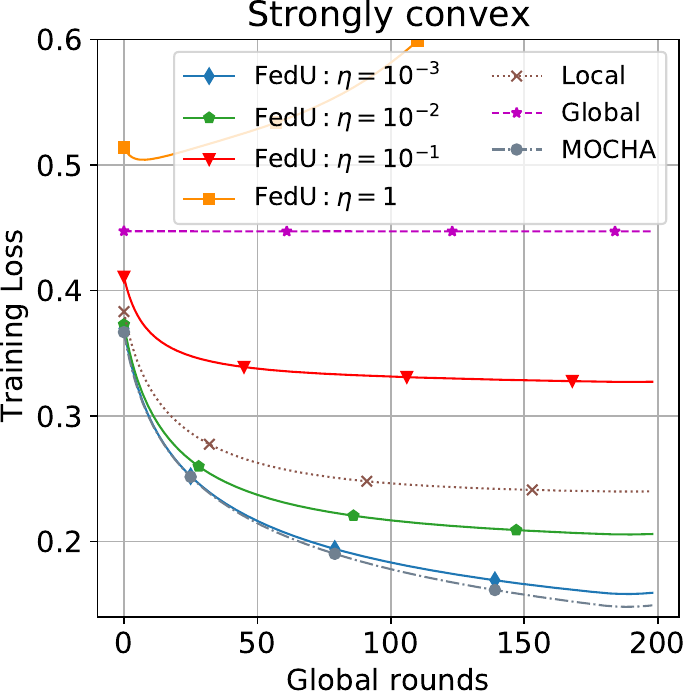}}\quad
		{\includegraphics[scale=0.355]{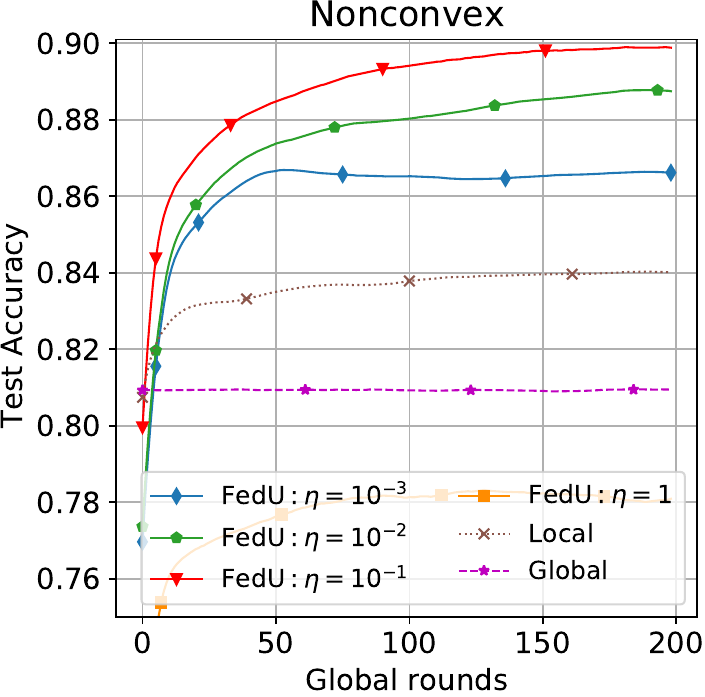}}\quad
		{\includegraphics[scale=0.355]{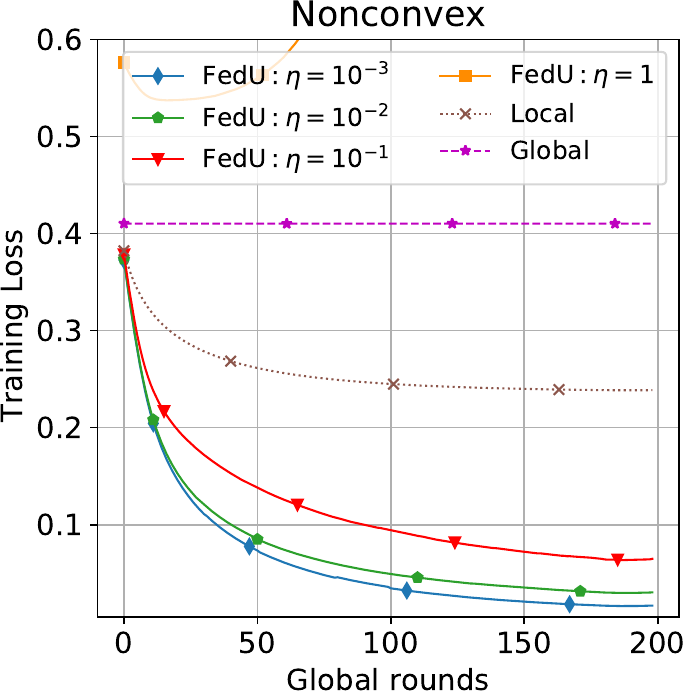}}
		\caption{Vehicle Sensor.}
		\label{F:Subeta_b}
	\end{subfigure}  
	\begin{subfigure}{1\textwidth} 
		\centering
		{\includegraphics[scale=0.355]{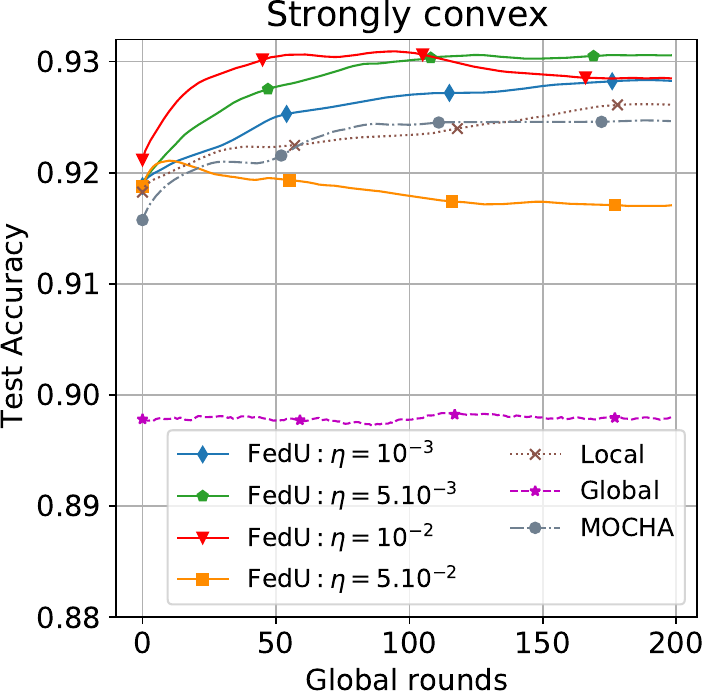}}\quad
		{\includegraphics[scale=0.355]{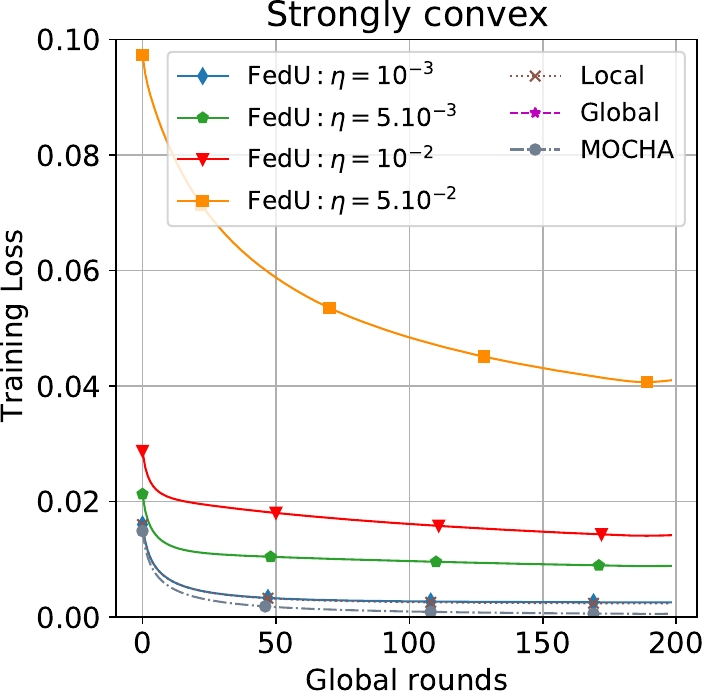}}\quad
		{\includegraphics[scale=0.355]{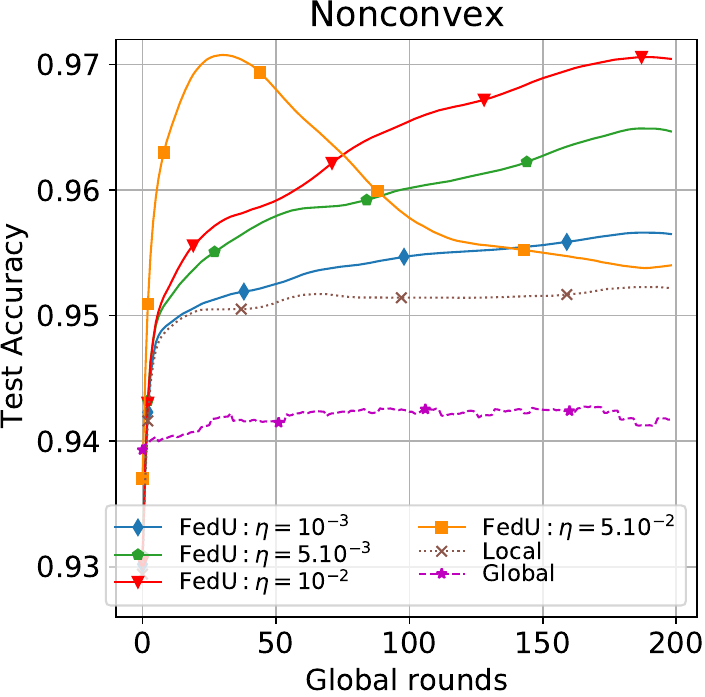}}\quad
		{\includegraphics[scale=0.355]{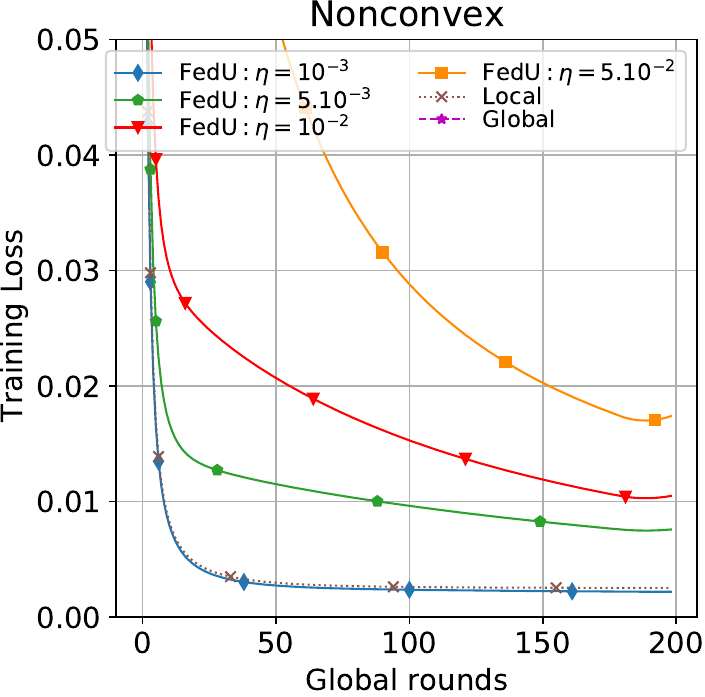}}
		\caption{MNIST.}
		\label{F:Subeta_c}
	\end{subfigure}
	\caption{Performance comparison between MOCHA, local model, global model, and \OurAlg  with the various sets of $\eta$ in both strongly convex and nonconvex settings.} \label{F:SubEta}
\end{figure*}
\begin{figure*}[!t]
	\centering
	\begin{subfigure}{1\textwidth} 
		\centering
		{\includegraphics[scale=0.355]{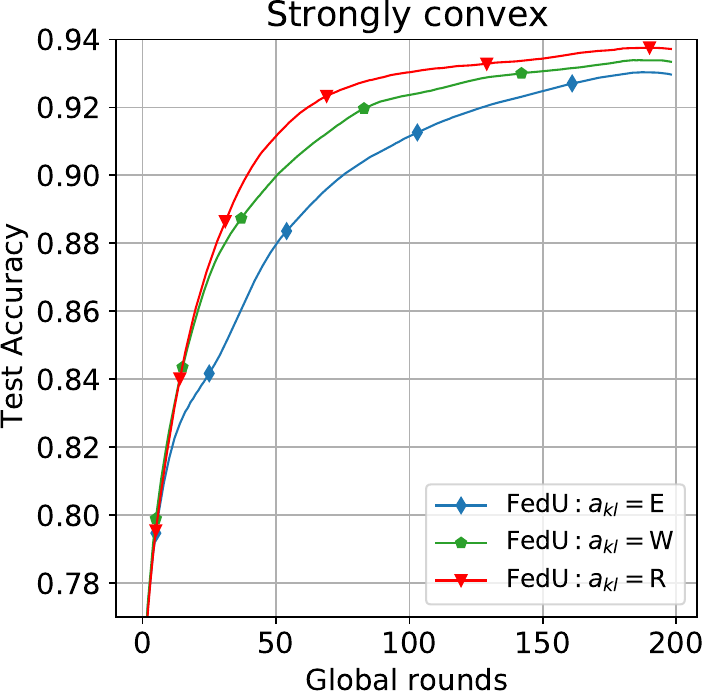}}\quad
		{\includegraphics[scale=0.355]{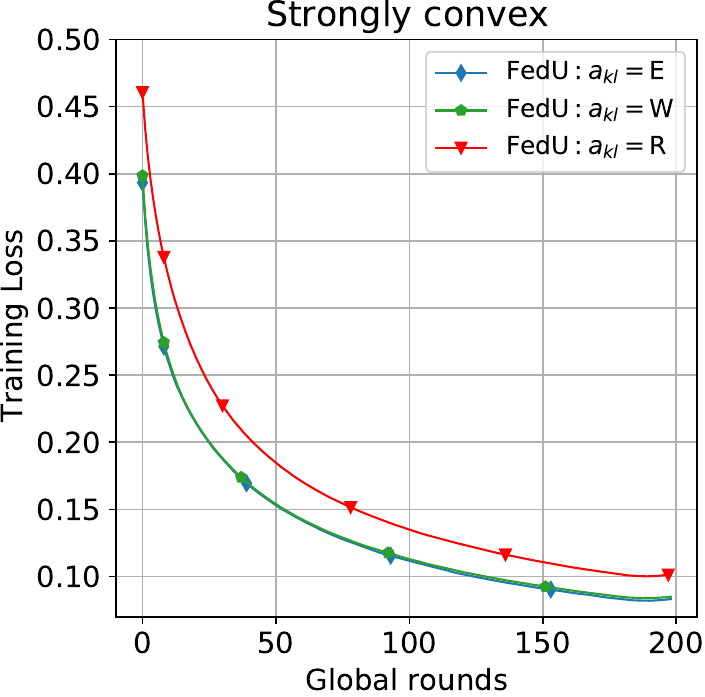}}\quad
		{\includegraphics[scale=0.355]{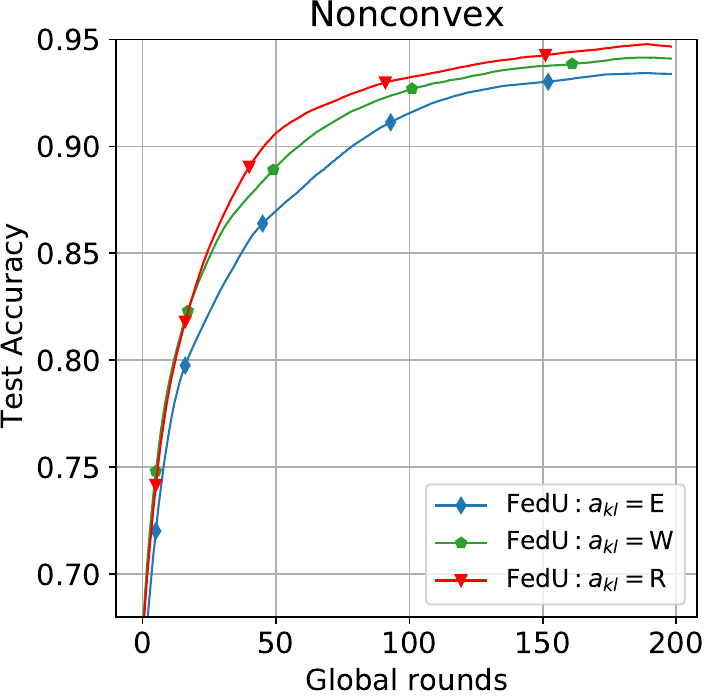}}\quad
		{\includegraphics[scale=0.355]{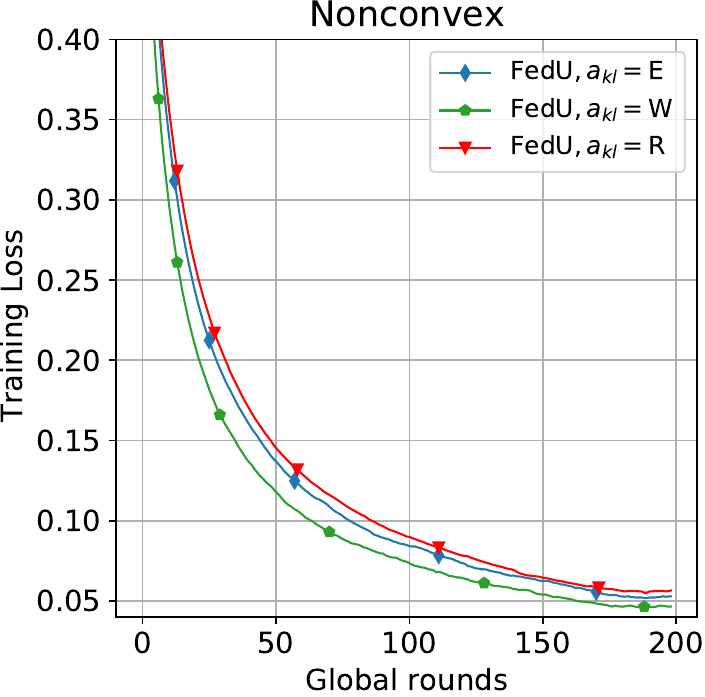}}
		\caption{Human Activity.}
		\label{F:Subakl_a}
	\end{subfigure}
	\begin{subfigure}{1\textwidth} 
		\centering
		{\includegraphics[scale=0.355]{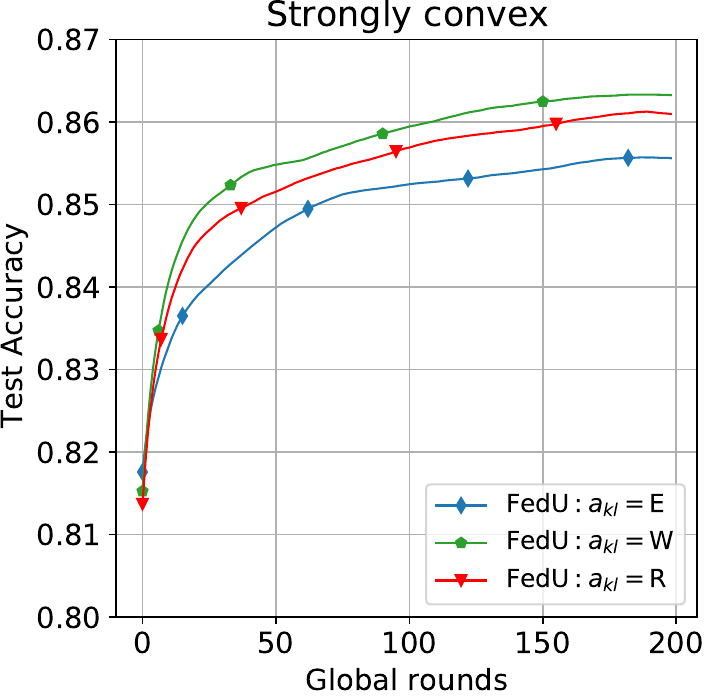}}\quad
		{\includegraphics[scale=0.355]{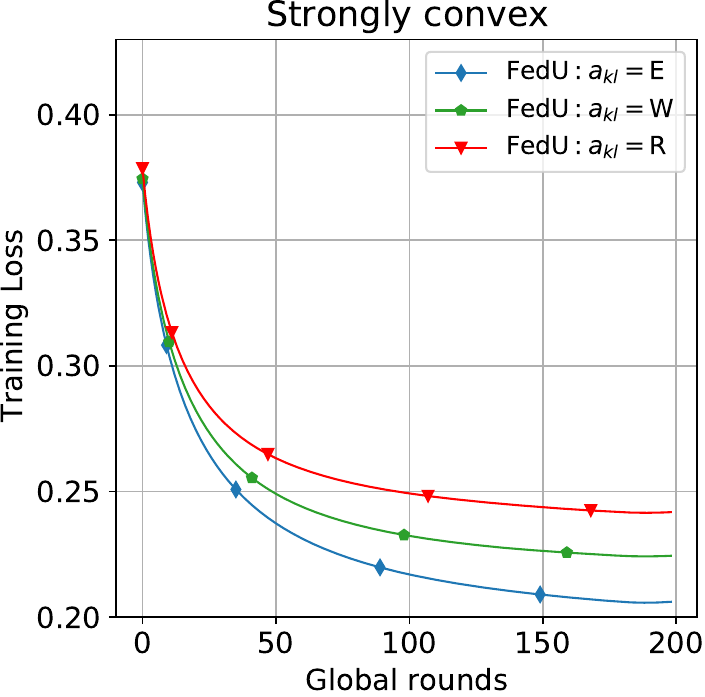}}\quad
		{\includegraphics[scale=0.355]{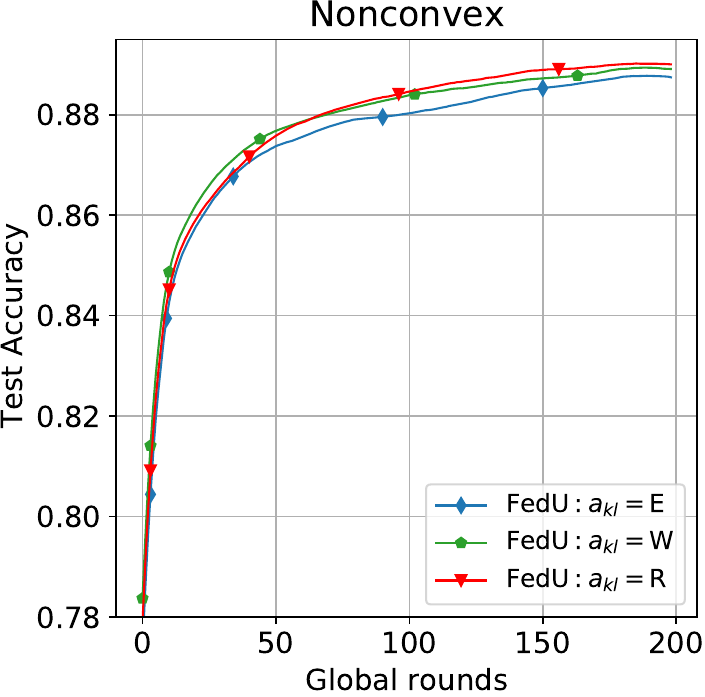}}\quad
		{\includegraphics[scale=0.355]{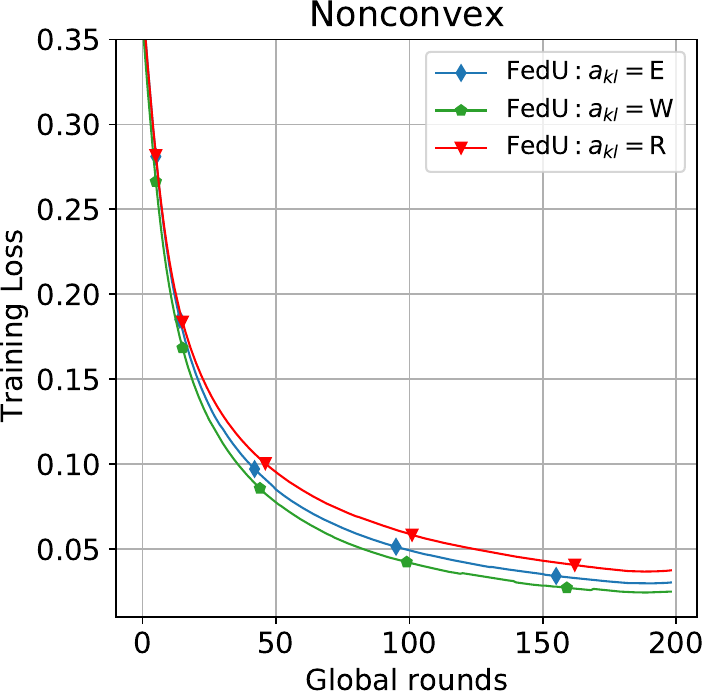}}
		\caption{Vehicle Sensor.}
		\label{F:Subakl_b}
	\end{subfigure}  
	\begin{subfigure}{1\textwidth} 
		\centering
		{\includegraphics[scale=0.355]{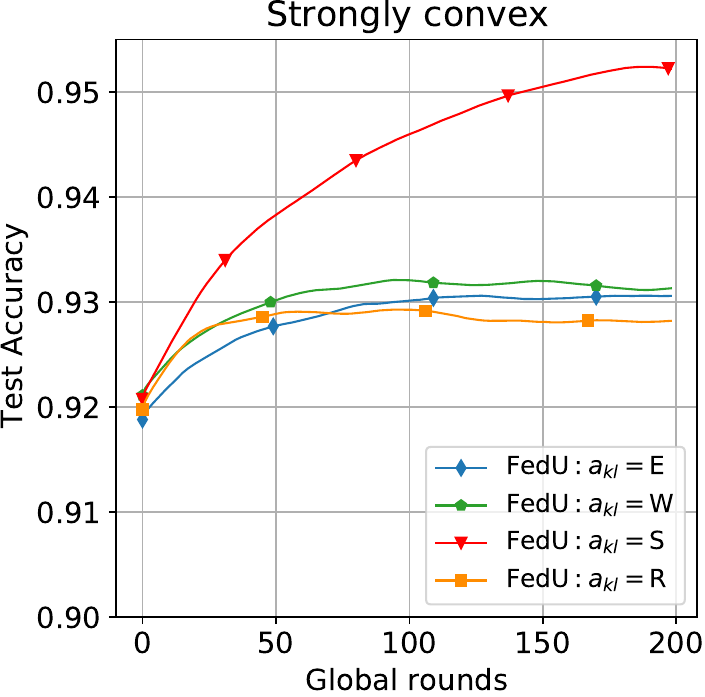}}\quad
		{\includegraphics[scale=0.355]{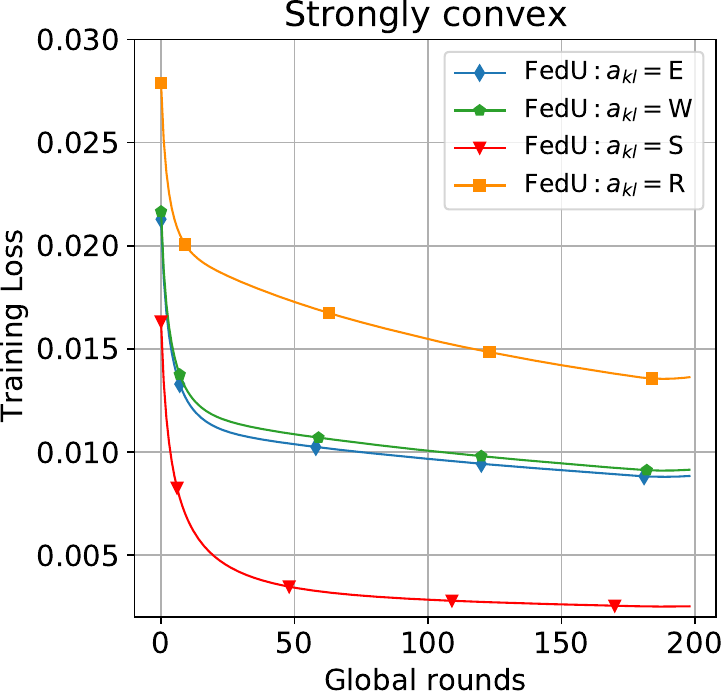}}\quad
		{\includegraphics[scale=0.355]{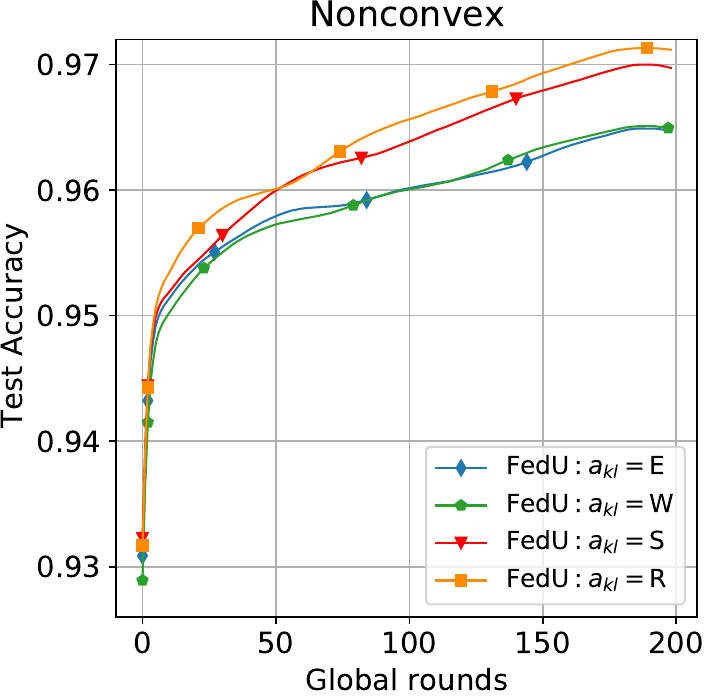}}\quad
		{\includegraphics[scale=0.355]{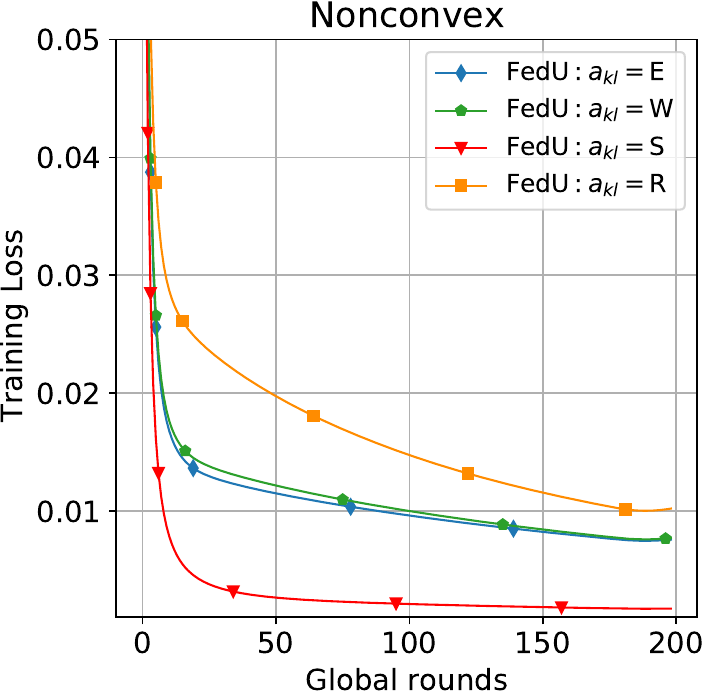}}
		\caption{MNIST.}
		\label{F:Subakl_c}
	\end{subfigure}
	\caption{Effects of graph information $\{a_{kl}\}$ on the convergence of \OurAlg in both convex and nonconvex settings.} 	\label{F:Subakl}
\end{figure*}

We first show benefits of  \OurAlg in FMTL setting by comparing \OurAlg with  Local model (training one separate model per client), Global model (training one single model on centralized data), and MOCHA, the conventional FMTL algorithm \cite{smith17NIPS}. 
Note that the performance results of the FMTL algorithm in \cite{yasmin21arXiv} and MOCHA are reported similar.
We evaluate  \OurAlg on a wide range of values of $\eta \in \{5.10^{-3},10^{-3},5.10^{-2},10^{-2},10^{-1}, 1\}$ and compare with others using their best fine-tuned parameters. In FMTL, each client represents a separate task. All clients have the same weight connection $\{a_{kl}\}$ with others and no client sampling in order to make fair comparisons with Local, Global models, and MOCHA. We also provide details on how to choose the different values of $\{a_{kl}\}$ in supplementary material. We only report the convex setting for MOCHA according to its assumption as stated in Section 3.1 \cite{smith17NIPS}.

The results in Fig.~\ref{F:SubEta} show that, \OurAlg\ achieves the highest performance, followed by MOCHA, Local model, and Global model. While the Local model at individual client learns only its own data without any contribution from the model of other clients, the Global model only does a single task that is not well generalized on highly non-i.i.d  data. We also recognize that Local model suffers overfitting when the data size at clients is small. By contrast, MOCHA and \OurAlg\ have the ability to learn models for multiple related tasks simultaneously and capture relationships amongst clients. Especially in the case of \OurAlg, using Laplacian regularization allows utilizing additional information about the structures of clients' models to increase the learning performance, and the contribution from clients having the large data size to those having smaller ones becomes more significant.

Observing different values of $\eta$, we found that the larger $\eta$ is, the more the coordination from other clients are, then \OurAlg performs better when $\eta$ is increased. However, when $\eta$ reaches a certain threshold, it slows down the convergence of \OurAlg, for example, $\eta = 5.10^{-2}$ in Fig.~\ref{F:SubEta}. $\eta$ then should be chosen carefully depends on the dataset. 

\subsection{Effect of the Graph Information $\{a_{kl}\}$ } 

For the above experiments, we assume that all relationships among a client and its neighbors are equal. However, in practice, the connection weights may have different values and they need to be known in advance. We then evaluate the effect of graph information shown in Fig.~\ref{F:Subakl} by normalizing the values of $\{a_{kl}\}$ in the range of  $[0, 1]$ and simulate 4 different scenarios of $\{a_{kl}\}$ as below: 
\begin{itemize}
	\item Random (R): All values of $\{a_{kl}\}$ are generated radomly $\{a_{kl}\} \thicksim \mathcal{N}(0, 1)$.
	\item Equal (E): When all clients have the same value for $\{a_{kl}\}$, we can choose any value of $\{a_{kl}\}$ in the range of  $[0, 1]$. However, there will be one value of $\eta* \{a_{kl} \}$ allows \OurAlg to achieve the highest accuracy. So, whenever  $\{a_{kl}\}$  is large, we can choose a small $\eta$, and vice versa. In this experiment, we fix $\{a_{kl}\} = 0.5$ and adjust $\eta$ accordingly. 
	
	\item Weighted (W): As there are various clients having significantly small data sizes, we set $\{a_{kl}\} = 0$ on the connection between these clients. We then set $\{a_{kl}\} = 0.5$ on the connection among clients having small data sizes and those having large data sizes, and  $\{a_{kl}\} = 1$ for all other connections.
	
	\item Similar (S): This scenario is only for MNIST. When distributing data to all clients, each client has 2 labels over 10.  Hence, clients may share only one, two similar labels or none of them. We set $\{a_{kl}\} = 0$, $\{a_{kl}\} = 0.5$, and $\{a_{kl}\} = 1$  for the connections among clients having no similar label, one similar label, and two similar labels, respectively.		
\end{itemize}

In most of the cases, the performance of \OurAlg with random $\{a_{kl}\}$ is better than that with equal $\{a_{kl}\}$.  When $\{a_{kl}\}$ are weighted, \OurAlg performs better than when all $\{a_{kl}\}$ are equal. Especially for MNIST, when $\{a_{kl}\}$ are weighted based on the similarity of clients, \OurAlg achieves the highest performance compared to other scenarios. Therefore, given knowing the relationship between client's data distribution, for example, in a weather forecasts application, clients in the same geographical location may have similar or close weather data, we can set higher values of weight connection for those clients than clients are in different locations to takes advantages of \OurAlg.

\subsection{Comparison with Personalized FL algorithms}

{Finally, we compare \OurAlg with the conventional FL algorithms FedAvg, FedProx, SCAFFOLD, AFL,  MOCHA, and with the state-of-the-art personalized FL algorithms pFedMe and Per-FedAvg.} The results are shown  in Table.~\ref{T:compare_performance}. We fix the subset of clients $S = 0.1N$ and perform the comparison on all four real datasets.
Overall, \OurAlg almost maintains the top performance in all scenarios. 

	\begin{table}[t!]
		\centering
		\caption{ Performance comparison of centralized setting ($R = 5$, $ S = 0.1 N$, $B = 20$, $T = 200$). There is no convex model for CIFAR-10, we then only report the non-convex case.}
			\begin{tabular}{l|c|ll}
				\multirow{3}{*}{Dataset} & \multirow{3}{*}{Algorithm} & \multicolumn{2}{c}{Test Accurancy}                                                                                        \\
				&                            & \multicolumn{1}{c}{Convex} & \multicolumn{1}{c}{Non Convex}\\ \hline
				\multirow{4}{*}{CIFAR-10 }      
				& \OurAlg            &                &     $\textbf{75.41}    \pm 0.29 $                          \\
				& pFedMe       &             &    $    74.10 \pm 0.89 $                \\
				& Per-FedAvg                 &     &      $64.70     \pm 1.91 $                             \\ 
				& FedAvg              	&  & $34.48          		\pm 5.34$            \\ 
				& {FedProx}             	&  & {$42.31          		\pm 4.21$       }     \\ 
				& {SCAFFOLD }            	&  & {$45.12          		\pm 3.38$          }  \\ 
				& {AFL}             	&  & {$49.07       		\pm 3.35$      }      \\
				\hline
				\multirow{4}{*}{MNIST}      
				& \OurAlg                  & $ \textbf{96.95}      \pm 0.11$  & $ 97.81      \pm 0.01$      \\
				& MOCHA           &       $ 96.18 \pm 0.09  $               \\	
				& pFedMe            &          $ 93.73    \pm 0.40$     &$\textbf{98.64} \pm 0.17$   \\
				& Per-FedAvg        &       $ 90.33     \pm 0.84  $ &     $ 96.38     \pm 0.40  $                \\
				& FedAvg            		&  $87.75         		\pm 1.31$   	&  $91.48         		\pm 1.05$               \\ 
				& {FedProx   }          	& {$88.70	\pm 1.18$ }  & {$91.60          		\pm 0.23$      }      \\ 
				& {SCAFFOLD}             	&{ $89.45          		\pm 0.37$ } & {$92.15          		\pm 0.43$     }       \\ 
				&{ AFL }            	&  {$89.79         		\pm 1.23$ } &{$92.01         		\pm 1.21$       }     \\
				\hline
				\multirow{4}{*}{\begin{tabular}[c]{@{}c@{}} Vehicle\\ Sensor\end{tabular}  }       
				& \OurAlg          &        $ \textbf{88.47}   \pm 0.21$   &      $ \textbf{91.79}   \pm 0.31$         \\
				& MOCHA           &    $  87.31  \pm 0.23 $        	 \\
				& pFedMe          &         $81.38 \pm 0.41$     	 	   &         $90.62 \pm 0.41$     		\\
				& Per-FedAvg    &         $81.07 \pm 0.71$       &         $86.92 \pm 1.3$           \\
				& FedAvg          &  $79.84  \pm 0.91 $         &  $84.04  \pm 2.69 $         \\ 
				& {FedProx }        & {$82.06 \pm 0.91$  }& {$87.65          		\pm 2.34$  }          \\ 
				& {SCAFFOLD }            	&{  $81.97 \pm 0.91$}  & {$88.48          		\pm 0.34$   }         \\ 
				& {AFL}             	& {$82.25 \pm 0.91$}  & {$87.88          		\pm 1.08$   }         \\
				\hline
				\multirow{4}{*}{\begin{tabular}[c]{@{}c@{}} Human\\ Activity\end{tabular}}       
				& \OurAlg             &  $\textbf{95.76}      \pm 0.46$  &       $\textbf{95.86 }     \pm 0.36$           	 \\
				& MOCHA       &          $ 92.33      \pm 0.67 $              	\\
				& pFedMe                   &          $ 95.41      \pm 0.38$   &          $ 95.72      \pm 0.32$   		\\
				& Per-FedAvg        &       $ 94.78      \pm 0.37  $    &       $ 94.80      \pm 0.60  $             \\
				& FedAvg              	&  $93.41           		\pm 0.95$ &  $93.74           		\pm 1.01$    \\ 
				& {FedPro}             	&  {$93.69       		\pm 0.84$  }  & {$94.65       		\pm 0.72$   }         \\ 
				& {SCAFFOLD}          &{$93.61  \pm 0.37  $  }    & {$94.78          		\pm 0.85$ }           \\ 
				& {AFL }            	        &{$93.92          		\pm 0.34$}     & {$94.42          		\pm 0.34$     }       \\\hline
			\end{tabular}
			\label{T:compare_performance}	
		\end{table}

%% file: analysis.tex
\setcounter{lemma}{1}
\setcounter{equation}{10}
\renewcommand\theenumi{(\alph{enumi})}
\renewcommand{\labelenumi}{\rm (\alph{enumi})}

\subsection{Technicalities}
In this section, we introduce additional definitions and technical lemmas which will be useful for our analysis of \OurAlg.
\begin{align}
\label{GDF}
& \nabla F(W) := [\nabla_{w_{1}} F (W)^T,\dots,\nabla_{w_{N}} F (W)^T]^T
\!= [\nabla F_1(w_{1})^T,\dots,\nabla F_N(w_{N})^T]^T \!
\in\RRR^{dN}\,\text{is the gradient of}\,F(W)
\\
\label{GDJ}
& \nabla J(W) := [\nabla_{w_{1}} J(W)^T,\dots,\nabla_{w_{N}} J(W)^T]^T
\overset{\eqref{mainP:FMTL}}{=}
\nabla F(W) + \eta\LL W
\in\RRR^{dN}\,\text{is the gradient of}\,J(W)
\\
& \zeta = \{\zeta_1,\dots,\zeta_N\}\,\text{is the set of random samples of clients}
\\
& \nabla \widetilde{F}(W,\zeta) :=
[\nabla \widetilde{F}_1(w_{1},\zeta_1)^T,\dots,\nabla \widetilde{F}_N(w_{N},\zeta_N)^T]^T
\in\RRR^{dN}\,\text{is the stochastic gradient of}\,F(W)
\\
\label{SGDJ}
& \nabla \widetilde{J}(W,\zeta) :=
\nabla \widetilde{F}(W,\zeta) + \eta\LL W^{(t)}
\in\RRR^{dN}\,\text{is the stochastic gradient of}\,J(W)
\\
& \widehat{S}^{(t)} = [s_1^{(t)},\dots,s_N^{(t)}] \in\RRR^N \text{is a client sampling random vector at round $t$, where }
s_k^{(t)} =
\begin{cases}
  1, & \mbox{if $k\in\SSS^{(t)}$} \\
  0, & \mbox{otherwise}
\end{cases}
\\
\label{St}
& \widetilde{S}^{(t)} = \diag(\widehat{S}^{(t)})\otimes I_d \in\RRR^{dN\times dN} \text{is a sampling matrix}
\\
& \LL = L\otimes I_d \in\RRR^{dN\times dN}
\\
\label{C}
& C =  I - \mut\eta\widetilde{S}\LL \in\RRR^{dN\times dN}\,\,\text{is a server-update matrix}
\\
\label{tau}
& \tau = \frac{S}{N}\,\,\text{is a client sampling factor}.
\end{align}
In what follows, $\|\cdot\|$ represents the $2$-norm for matrix and the Euclidean norm for vector.

For a connected graph $\GG$, $L=D-A$ is a symmetric positive semi-definite matrix with $\lambda_{\min}(L)=\lambda_1=0<\lambda_2\leq \dots \leq\lambda_N=\lambda_{\max}(L)=\rho$, in order. As such, the matrix $\LL=L\otimes I_d$ has  $\lambda_{\min}(\LL)=\lambda_{\min}(L)\lambda_{\min}(I_d)=\lambda_{\min}(L)=0$ and $\lambda_{\max}(\LL)=\lambda_{\max}(L)\lambda_{\max}(I_d)=\rho$.
When no client sampling, $\widetilde{S} = I_{dN}$, the matrix $C=I-\mut\eta\LL$  has $\lambda_{\min}(C)=1-\mut\eta\lambda_{\max}(\LL)=1-\mut\eta\rho$ and $\lambda_{\max}(C)=1-\mut\eta\lambda_{\min}(\LL)=1$.
Since $C$ is symmetric, we have $\|C\|=\max \{|\lambda|: \text{$\lambda$ is an eigenvalue of $C$}\}$.
Therefore, $C$ (when no client sampling) is normalized (i.e., $\|C\|^2 = 1$) if and only if
\begin{align}
\label{normalserverupdate}
\mut\eta\rho\leq 2.
\end{align}


\begin{lemma}[Sampling matrix's properties]
\label{Sproperty}
Let $\widetilde{S}$ be defined as $\widetilde{S}^{(t)}$ in \eqref{St}. Then
\begin{enumerate}
\item\label{Sproperty_norm} $\|\widetilde{S}\|=1$;
\item\label{Sproperty_ST} $\widetilde{S}^T = \widetilde{S}$;
\item\label{Sproperty_S2} $\widetilde{S}\widetilde{S} = \widetilde{S}^T\widetilde{S} = \widetilde{S}$;
\item\label{Sproperty_ES} $\EEE \widetilde{S} = \tau I_{dN}$;
\item\label{Sproperty_normSY} $\EEE\|\widetilde{S}Y\|^2=\tau\EEE\|Y\|^2, \forall Y\in\RRR^{dN}$
\end{enumerate}
\end{lemma}
\begin{proof}
\ref{Sproperty_norm}--\ref{Sproperty_ES} follow directly from the definition of $\widetilde{S}$, while \ref{Sproperty_normSY} from the fact that $\EEE\|\widetilde{S}Y\|^2=\EEE Y^T\widetilde{S}^T\widetilde{S}Y
\overset{\text{\ref{Sproperty_S2}}}{=}
\EEE Y^T\widetilde{S}Y
\overset{\text{\ref{Sproperty_ES}}}{=}
\tau\EEE Y^TY=\tau\EEE\|Y\|^2$.
\end{proof}

\begin{lemma}[Jensen's inequality]
\label{Jensen}
For any vector $X_i\in\RRR^{dN}, i\in\{1,\dots,M\}$,
\begin{align}
\left\|\sum_{i=1}^M X_i\right\|^2\leq M \sum_{i=1}^{M} \|X_i\|^2.
\end{align}
\end{lemma}

\begin{lemma} [Young inequality]
\label{Young}
For any vector $X,Y\in\RRR^{dN}$ and $m>0$,
\begin{enumerate}
\item\label{Young_XY} $\langle X,Y\rangle
\leq \frac{m}{2}\|X\|^2 + \frac{1}{2m}\|Y\|^2$;
\item\label{Young_sumXY} $\|X+Y\|^2
\leq (1+m)\|X\|^2 + \left(1 + \frac{1}{m}\right)\|Y\|^2$.
\end{enumerate}
\end{lemma}

\begin{lemma}[Smoothness]
\label{lemma:Jsmooth}
Suppose that Assumption~\ref{assump:smooth} holds. Set $\beta_J :=\beta+\eta\rho$ with $\rho :=\|\LL\|$. Then, for any $W,W'\in\RRR^{dN}$,
\begin{enumerate}
\item\label{lemma:Jsmooth_distanceGDF} $\|\nabla F(W)-\nabla F(W')\|\leq \beta \|W-W'\|$;
\item\label{lemma:Jsmooth_distanceGDJ} $\|\nabla J(W)-\nabla J(W')\|\leq \beta_J \|W-W'\|$;
\item\label{lemma:Jsmooth_normGDF2}
$\|\nabla F(W)\|^2\leq 2\beta^2\|W-W'\|^2 +2\|\nabla F(W')\|^2$;
\item\label{lemma:Jsmooth_normGDJ2}
$\|\nabla J(W)\|^2\leq 2\beta_J^2\|W-W'\|^2 +2\|\nabla J(W')\|^2$;
\item\label{lemma:Jsmooth_distanceJ} $J(W) - J(W') \leq \langle \nabla J(W'),W-W' \rangle + \frac{\beta_J}{2}\|W-W'\|^2$.
\end{enumerate}
\end{lemma}
\begin{proof}
\ref{lemma:Jsmooth_distanceGDF}: This directly follows from Assumption~\ref{assump:smooth} and the definition of $\nabla F(W)$ in \eqref{GDF}.

\ref{lemma:Jsmooth_distanceGDJ}: Since $\nabla J(W) = \nabla F(W) +\eta\LL W$ and $\rho =\|\LL\|$, the conclusion follows from \ref{lemma:Jsmooth_distanceGDF}.

\ref{lemma:Jsmooth_normGDF2}: Using Lemma~\ref{Jensen}, we have
\begin{align}
\|\nabla F(W)\|^2 = \|\nabla F(W)-\nabla F(W') + \nabla F(W')\|^2
\leq 2\|\nabla F(W)-\nabla F(W')\|^2 + 2\|\nabla F(W')\|^2,
\end{align}
which together with \ref{lemma:Jsmooth_distanceGDF} implies \ref{lemma:Jsmooth_normGDF2}.

\ref{lemma:Jsmooth_normGDJ2}: The proof is similar to \ref{lemma:Jsmooth_normGDF2}.

\ref{lemma:Jsmooth_distanceJ}: This follows from Lemma~1.2.3 in \cite{nesterov18book}.
\end{proof}

\begin{lemma}[Strong convexity]
\label{lemma:Jstrongconvex}
Suppose that Assumption~\ref{assump:strongconvex} holds. Then, for any $W,W'\in\RRR^{dN}$,
\begin{align}
J(W)\geq J(W')\!+\!\left\langle\nabla J(W'),W-W'\right\rangle
\!+\!\frac{\alpha}{2}\|W-W'\|^2.
\end{align}
\end{lemma}
\begin{proof}
It follows from Assumption~\ref{assump:strongconvex} that $F$ is $\alpha$-strongly convex with respect to $W$. Since $J(W)=F(W)+\eta W^T\LL W$ and $\LL$ is a positive semi-definite matrix, we derive that $J$ is also $\alpha$-strongly convex with respect to $W$, and the conclusion follows.
\end{proof}

\begin{lemma}[Smoothness and strong convexity]
\label{lemma:Jsmoothconvex}
Suppose that Assumptions~\ref{assump:smooth} and~\ref{assump:strongconvex} hold. Let $W^*$ be the optimal solution to \eqref{mainP:FMTL}. Then, for any $W,W',W''\in\RRR^{dN}$,
\begin{enumerate}
\item\label{lemma:Jsmoothconvex_GDJ} $\|\nabla J(W)\|^2 \leq 2\beta_J[J(W)-J(W^*)]$;
\item\label{lemma:Jsmoothconvex_3W} $\left\langle\nabla J(W), W''- W'\right\rangle \geq J(W'') - J(W') + \frac{\alpha}{4}\|W'-W''\|^2 - \beta_J\|W''-W\|^2$;
\item\label{lemma:Jsmoothconvex_GDF} $\|\nabla F(W)\|^2 \leq 4\frac{\beta ^2}{\alpha}[J(W)-J(W^*)] + 2\sigma_2^2$,
where $\sigma_2$ is defined as in Lemma~\ref{lemma:Boundedgradient}.
\end{enumerate}
\end{lemma}
\begin{proof}
\ref{lemma:Jsmoothconvex_GDJ} is from Lemmas~\ref{lemma:Jsmooth}(b),~\ref{lemma:Jstrongconvex}, and Theorems~2.1.5 in \cite{nesterov18book}, while \ref{lemma:Jsmoothconvex_3W} is from Lemma~5 in \cite{sai20ICML}.

\ref{lemma:Jsmoothconvex_GDF}: Applying Lemma~\ref{lemma:Jsmooth}\ref{lemma:Jsmooth_normGDF2}, we have
\begin{align}
\label{nablaF2bound}
\|\nabla F(W)\|^2
\leq 2\beta^2\|W-W^*\|^2 + 2\|\nabla F(W^*)\|^2.
\end{align}
It follows from Theorem~2.1.8 in \cite{nesterov18book} that
\begin{align}\label{boundW2}
\|W-W^*\|^2 \leq \frac{2}{\alpha}[J(W)-J(W^*)].
\end{align}
By the definition of $\sigma_2$ in Lemma~\ref{lemma:Boundedgradient}, $\|\nabla F(W^*)\|^2 \leq \sigma_2^2$, which, together with \eqref{nablaF2bound} and \eqref{boundW2} completes the proof.
\end{proof}

\begin{lemma}[Bounded variance]
\label{lemma:boundvariance}
Suppose that Assumption~\ref{assump:boundvariance} holds. Then, for any $W\in\RRR^{dN}$,
\begin{enumerate}
\item\label{lemma:boundvariance_F} $\EEE_{\zeta}\|\nabla\widetilde{F}(W,\zeta)-\nabla F(W)\|^2\leq \sigma_1^2$;
\item\label{lemma:boundvariance_J} $\EEE_{\zeta}\|\nabla\widetilde{J}(W,\zeta)-\nabla J(W)\|^2\leq \sigma_1^2$.
\end{enumerate}
\end{lemma}
\begin{proof}
\ref{lemma:boundvariance_F} directly follows from Assumption~\ref{assump:boundvariance}. By the definitions of $\nabla J$ and $\nabla\widetilde{J}$ in \eqref{GDJ} and \eqref{SGDJ}, \ref{lemma:boundvariance_F} implies \ref{lemma:boundvariance_J}.
\end{proof}



\begin{lemma}
\label{lemma:serverupdate}
Let $\{\Xt_1,\dots,\Xt_r,\dots,\Xt_R\}$ be $R$ random variables in $\RRR^{dN}$ which are not necessarily independent. Suppose each $\Xt_r$ has a conditional mean $\EEE[\Xt_r|\Xt_{r-1},\dots,\Xt_1]=X_r$ (i.e., $\{\Xt_r-X_r\}$ form a martingale difference sequence), and a variance $\EEE\|\Xt_r-X_r\|^2\leq \sigma^2$. Then
\begin{align}
\label{boundavgSX}
\EEE\left\|\frac{1}{R}\sum_{r=0}^{R-1}\widetilde{S}\Xt_r\right\|^2
&\leq \frac{\tau}{R}\sum_{r=0}^{R-1}\EEE\left\|X_r\right\|^2
+ \frac{\tau\sigma^2}{R}.
\end{align}
where $\widetilde{S}$ is defined as $\widetilde{S}^{(t)}$ in \eqref{St}.
\end{lemma}
\begin{proof}
We see that
\begin{align}
\label{varsumx}
\EEE\left\|\sum_{r=0}^{R-1}(\Xt_r-X_r)\right\|^2
= \sum_{r=0}^{R-1}\EEE\|\Xt_r-X_r\|^2
+ \sum_{r,i}\EEE(\Xt_r-X_r)^T(\Xt_i-X_i)
= \sum_{r=0}^{R-1}\EEE\|\Xt_r-X_r\|^2,
\end{align}
where $\sum_{r,i}\EEE(\Xt_r-X_r)^T(\Xt_i-X_i)=0$ because $\{\Xt_r-X_r\}$ form a martingale difference sequence. On the other hand,
\begin{align}
\EEE\left\|\sum_{r=0}^{R-1}(\Xt_r-X_r)\right\|^2
= \EEE\left\|\sum_{r=0}^{R-1}\Xt_r\right\|^2 -
\left\|\sum_{r=0}^{R-1}X_r\right\|^2,
\end{align}
which, together with \eqref{varsumx}, implies that
\begin{align}
\label{esumx2}
\EEE\left\|\sum_{r=0}^{R-1}\Xt_r\right\|^2
= \left\|\sum_{r=0}^{R-1}X_r\right\|^2
+ \sum_{r=0}^{R-1}\EEE\|\Xt_r-X_r\|^2.
\end{align}
Multiplying both sides of \eqref{esumx2} by $\frac{1}{R^2}$ and using $\EEE\|\Xt_r-X_r\|^2\leq \sigma^2, \forall r$, we get
\begin{align}
\label{boundavgX}
\EEE\left\|\frac{1}{R}\sum_{r=0}^{R-1}\Xt_r\right\|^2
\leq \EEE\left\|\frac{1}{R}\sum_{r=0}^{R-1}X_r\right\|^2 + \frac{\sigma^2}{R}.
\end{align}
This together with Lemma~\ref{Sproperty}\ref{Sproperty_ES} yields
\begin{align}
\EEE\left\|\frac{1}{R}\sum_{r=0}^{R-1}\widetilde{S}\Xt_r\right\|^2
\leq \tau\EEE\left\|\frac{1}{R}\sum_{r=0}^{R-1}\Xt_r\right\|^2
\leq \tau\EEE\left\|\frac{1}{R}\sum_{r=0}^{R-1}X_r\right\|^2
+ \frac{\tau\sigma^2}{R}
\leq \frac{\tau}{R}\sum_{r=0}^{R-1}\EEE\left\|X_r\right\|^2
+ \frac{\tau\sigma^2}{R},
\end{align}
which completes the proof.
\end{proof}

\begin{lemma}[]
\label{lemma:clientupdate}
Let $\Xt$ is a random variable in $\RRR^{dN}$ with mean $\EEE \Xt = X$ and variance $\EEE\|\Xt-X\|^2\leq \sigma^2$. Let $\widetilde{S}$ be defined as $\widetilde{S}^{(t)}$ in \eqref{St}. Then, for any $\mu\geq 0$ and $Y\in\RRR^{dN}$,
\begin{align}
\label{boundedupdate}
\EEE\|Y-\mu\widetilde{S}\Xt\|^2
\leq (1-\tau)\EEE\|Y\|^2
+ \tau\EEE\|Y-\mu X\|^2
+ \mu^2\tau\sigma^2.
\end{align}
\end{lemma}
\begin{proof}
On one hand, by Lemma~\ref{Sproperty}\ref{Sproperty_S2},
\begin{align}
\nonumber
\EEE\|Y-\mu\widetilde{S}\Xt\|^2
&= \EEE (Y-\mu\widetilde{S}\Xt)^T(Y-\mu\widetilde{S}\Xt)
\\
\nonumber
& = \EEE(Y^TY-\mu\Xt^T\widetilde{S}Y - \mu Y^T\widetilde{S}\Xt + \mu^2\Xt^T\widetilde{S}\Xt)
\\
\label{EYSX}
& = \EEE(Y^TY-\mu\tau X^TY - \mu \tau Y^TX + \mu^2 \tau \Xt^T\Xt)
\end{align}
On the other hand, $\EEE\Xt^T\Xt \leq  X^TX + \sigma^2$ since $\EEE\|\Xt-X\|^2=\EEE\Xt^T\Xt - X^TX\leq \sigma^2$. Therefore,
\begin{align}
\nonumber
\EEE\|Y-\mu\widetilde{S}\Xt\|^2
& \leq \EEE(Y^TY-\mu \tau X^TY - \mu \tau Y^TX + \mu^2\tau X^TX + \mu^2 \tau\sigma^2)
\\
\nonumber
& = (1-\tau)\EEE(Y^TY)+\tau\EEE(Y^TY-\mu X^TY - \mu Y^T\Xt + \mu^2 X^TX) + \mu^2\tau\sigma^2
\\
& = (1-\tau)\EEE\|Y\|^2
+ \tau\EEE\|Y-\mu X\|^2
+ \mu^2\tau\sigma^2,
\end{align}
which finishes the proof.
\end{proof}

\subsection{Proof of Lemma~\ref{lemma:Boundedgradient}}
It follows from the definition of $J$ and $W$ that \eqref{sigma3main} can be written as
\begin{align}\label{sigma3}
\|\nabla F(W)\|^2
\leq \sigma_2^2 + \|\nabla J(W)\|^2
= \sigma_2^2 + \|\nabla F(W) + \eta\LL W\|^2,
\end{align}
which is equivalent to
\begin{align}
\label{existtarget}
\AAA := -\eta^2\rho^2\|W\|^2 -2\eta\langle\nabla F(W),\LL W\rangle \leq \sigma_2^2.
\end{align}
By Assumption~\ref{assump:smooth} and the definition of $\nabla F(W)$,
\begin{align}
\|\nabla F(W) - \nabla F(0)\|
\leq
\beta\|W-0\| = \beta\|W\|,
\end{align}
which implies that
\begin{align}
\label{betaWpsi}
\|\nabla F(W)\| \leq \beta \|W\| + \|\nabla F(0)\|.
\end{align}
Combining with Cauchy--Schwartz inequality, we obtain that
\begin{align}
|\langle\nabla F(W),\LL W\rangle|
\leq \|\nabla F(W)\|\|\LL W\|
\leq \beta\rho\|W\|^2 + \psi\rho\|W\|,
\end{align}
where $\psi :=\|\nabla F(0)\|$. It follows that
\begin{align}
\label{nablaFL}
-\langle\nabla F(W),\LL W\rangle
\leq \beta\rho\|W\|^2 +\psi\rho\|W\|,
\end{align}
and so
\begin{align}
\nonumber
\AAA
&\leq -\eta^2\rho^2\|W\|^2 +2\eta(\beta\rho\|W\|^2 +\psi\rho\|W\|) = -\eta\rho(\eta\rho-2\beta)\|W\|^2
+2\psi\eta\rho\|W\| \\ \nonumber
& = -\eta\rho(\eta\rho-2\beta)\left(\|W\|^2-\frac{\psi}{\eta\rho-2\beta}\right)^2 +\frac{\psi^2\eta\rho}{\eta\rho-2\beta} \\
&\leq \frac{\psi^2\eta\rho}{\eta\rho-2\beta},
\end{align}
where the last inequality is due to the assumption that $\eta\rho > 2\beta$. Therefore, \eqref{existtarget} always holds if $\sigma_2^2\geq\frac{\psi^2\eta\rho}{\eta\rho-2\beta}$.

\subsection{Analysis of \OurAlg}
For ease of analysis, we rewrite Algorithm~\ref{alg:2SSGD} as Algorithm~\ref{alg:2SSGD:matrix} with matrix notations. Here, Line $5$ of Algorithm~\ref{alg:2SSGD:matrix} represents Lines $9$ and $14$ of Algorithm~\ref{alg:2SSGD}, while Line $7$ of Algorithm~\ref{alg:2SSGD:matrix} represents Lines $14$ and $15$ of Algorithm~\ref{alg:2SSGD}.

\begin{algorithm}[!h]
\caption{\OurAlg with Matrix Notation}
\begin{algorithmic}[1]\label{alg:2SSGD:matrix}
\STATE \textbf{server's input}: initial $W^{(0)}$
\FOR{each round $t = 0,\dots,T-1$}
\FOR{$r = 0,\dots,R-1$}
\STATE initialize $W_0^{(t)}\leftarrow W^{(t)}$
\STATE $W_{r+1}^{(t)} =  W_r^{(t)} - \mu \widetilde{S}^{(t)}\nabla \widetilde{F}(W_r^{(t)})$
\ENDFOR
\STATE $W^{(t+1)}=(I-\mut\eta \widetilde{S}^{(t)}\LL) W_R^{(t)}= C^{(t)}W_R^{(t)}$
\ENDFOR
\end{algorithmic}
\end{algorithm}

In round $t$, the local update
\begin{align}\label{}
  W_{r+1}^{(t)} = W_{r}^{(t)} - \mu\widetilde{S}^{(t)}\nabla \widetilde{F}(W_{r}^{(t)})
\end{align}
implies that after $R$ local update steps, we have
\begin{align}\label{gt}
\mu\widetilde{S}^{(t)}\sum_{r=0}^{R-1}\nabla\widetilde{F}(W_{r}^{(t)}) = \sum_{r=0}^{R-1}(W_{r}^{(t)}-W_{r+1}^{(t)})
=  W_{0}^{(t)} -  W_{R}^{(t)}
=  W^{(t)}-W_R^{(t)}.
\end{align}
We then rewrite the server update as follows
\begin{align}
\label{Wnew}
\nonumber
W^{(t+1)}
&
= C^{(t)} W_R^{(t)}
\overset{\eqref{gt}}{=}
C^{(t)}\left[W^{(t)} - \mu  R \frac{1}{R}\sum_{r=0}^{R-1}\widetilde{S}^{(t)}\nabla\widetilde{F}(W_{r}^{(t)})\right]
\\
\nonumber
&\overset{\eqref{C}}{=}
(I-\mut\eta\widetilde{S}^{(t)} \LL)\left[W^{(t)}
- \frac{\mut}{R}\sum_{r=0}^{R-1}\widetilde{S}^{(t)}\nabla\widetilde{F}(W_{r}^{(t)})\right]
\\
\nonumber
& =
W^{(t)}
- \frac{\mut\widetilde{S}^{(t)}}{R}\sum_{r=0}^{R-1}\nabla\widetilde{F}(W_{r}^{(t)})
- \mut\eta \widetilde{S}^{(t)}\LL W^{(t)}
+  \frac{\mut^2\eta\widetilde{S}^{(t)}\LL}{R}\sum_{r=0}^{R-1}
\widetilde{S}^{(t)}\nabla\widetilde{F}(W_{r}^{(t)})
\\
\nonumber
& =  W^{(t)}
\!-\! \frac{\mut}{R}\sum_{r=0}^{R-1}\widetilde{S}^{(t)}\left[\nabla\widetilde{F}(W_{r}^{(t)}) + \eta\LL W_r^{(t)}\right]
\!+\! \frac{\mut\eta}{R}\sum_{r=0}^{R-1}\widetilde{S}^{(t)}\LL(W_r^{(t)}-W^{(t)})
\!+\!  \frac{\mut^2\eta\widetilde{S}^{(t)}\LL}{R}\sum_{r=0}^{R-1}\widetilde{S}^{(t)}\nabla\widetilde{F}(W_{r}^{(t)})
\\
& = W^{(t)} - \mut Z^{(t)},
\end{align}
where
\begin{align}
Z^{(t)}=
\frac{1}{R}\sum_{r=0}^{R-1}\widetilde{S}^{(t)}\nabla\widetilde{J}(W_{r}^{(t)})
- \frac{\eta}{R}\sum_{r=0}^{R-1}\widetilde{S}^{(t)}\LL(W_r^{(t)}-W^{(t)})
- \frac{\mut\eta\widetilde{S}^{(t)}\LL}{R}\sum_{r=0}^{R-1}\widetilde{S}^{(t)}\nabla\widetilde{F}(W_{r}^{(t)}).
\end{align}
Finally, we output $\widetilde{W}^{(T)}=W^{(t)}$ with probability $\frac{\theta^{(t)}}{\sum_{t=0}^{T-1}\theta^{(t)}}$ for some weights $\theta^{(t)}$, and $r\in\{0,\dots,T-1\}$.

Let $\EE^{(t)} := \frac{1}{R}\sum_{r=0}^{R-1}\EEE\|W_r^{(t)}-W^{(t)}\|^2$ be the drift caused by $R$ local update steps at clients, where $\EEE$ is the expectation taken over all random sources. We now provide some supporting lemmas as follows.

\begin{lemma}[Bounded drift]
\label{lemma:boundeddrift}
Suppose that Assumption~\ref{assump:boundvariance} holds. Then
\begin{align}
\label{boundedriftsmooth}
\EE^{(t)}
\leq 4\mut^2\tau\EEE\|\nabla F(W^{(t)})\|^2
+ \frac{2\mut^2\tau\sigma_1^2}{ R}.
\end{align}
\end{lemma}
\begin{proof}
By Assumption~\ref{assump:boundvariance}, using Lemmas~\ref{lemma:boundvariance}\ref{lemma:boundvariance_F},~\ref{lemma:clientupdate} and then~\ref{Young}\ref{Young_sumXY}, we derive that
\begin{align}
\label{Term1}
\nonumber
&\EEE\|W_r^{(t)}-W^{(t)}\|^2 = \EEE\|W_{r-1}^{(t)}-W^{(t)}-\mu\widetilde{S}^{(t)}\nabla\widetilde{F}(W_{r-1}^{(t)})\|^2
\\
\nonumber
& \leq (1-\tau)\EEE\|W_{r-1}^{(t)}-W^{(t)}\|^2
+ \tau\EEE\|W_{r-1}^{(t)}-W^{(t)}-\mu\nabla F(W_{r-1}^{(t)})\|^2
+ \mu^2\tau\sigma_1^2
\\
\nonumber
&\leq
(1-\tau)\EEE\|W_{r-1}^{(t)}-W^{(t)}\|^2
+ \left(1+\frac{1}{R\tau}\right)\tau\EEE\|W_{r-1}^{(t)}-W^{(t)}\|^2
+ \left(1+R\tau\right)\mu^2\tau\EEE\|\nabla F(W^{(t)})\|^2
+ \mu^2\tau\sigma_1^2
\\
&\leq
\left(1+\frac{1}{R}\right)\EEE\|W_{r-1}^{(t)}-W^{(t)}\|^2
+ \frac{2\mut^2\tau}{ R}\EEE\|\nabla F(W^{(t)})\|^2
+ \frac{\mut^2\tau\sigma_1^2}{R^2},
\end{align}
where the last inequality is due to the fact that $1+R\tau\leq R+R =2R$ since $R\geq 1$ and $\tau\leq 1$. Telescoping \eqref{Term1} yields
\begin{align}\label{}
\EEE\|W_r^{(t)}-W^{(t)}\|^2
\leq \left(\frac{2\mut^2\tau}{R}\EEE\|\nabla F(W^{(t)})\|^2
+ \frac{\mut^2\tau\sigma_1^2}{R^2}\right)\sum_{r=1}^{R-1}\left(1+\frac{1}{R}\right)^r.
\end{align}
Since $\sum_{j=0}^{m-1}x_j=\frac{x^m-1}{x-1}$ and $\left(1+\frac{x}{n}\right)^n\leq e^x, \forall x\in\RRR, n\in\NNN$, we have $\sum_{r=0}^{R-1}\left(1+\frac{1}{R}\right)^r
=\frac{\left(1+\frac{1}{R}\right)^R-1}{\left(1+\frac{1}{R}\right)-1}\leq (e-1)R\leq2R$, and thus
\begin{align}\label{distanceWrW}
\EEE\|W_r^{(t)}-W^{(t)}\|^2
\leq 4\mut^2\tau\EEE\|\nabla F(W^{(t)})\|^2
+ \frac{2\mut^2\tau\sigma_1^2}{ R}.
\end{align}
Averaging \eqref{distanceWrW} over $r$, we get the conclusion.
%
\end{proof}

\begin{lemma}
\label{lemma:EZ}
Suppose that Assumptions~\ref{assump:smooth} and~\ref{assump:boundvariance} hold. Then
\begin{align}
\EEE\|Z^{(t)}\|^2
\!\leq \!
\tau(6\beta_J^2+3\eta^2\rho^2+6\mut^2\eta^2\rho^2\beta^2)\EE^{(t)}\!+ 6\tau \EEE\|\nabla J(W^{(t)})\|^2 \!
+\! 6\tau\mut^2\eta^2\rho^2\EEE\|\nabla F(W^{(t)})\|^2\!
+\! \frac{3\tau(1+\mut^2\eta^2\rho^2)\sigma_1^2}{R}.
\end{align}
\end{lemma}
\begin{proof}


Using Lemma~\ref{Jensen}, we have that
\begin{align}
\label{EZ2}
\nonumber
\EEE\| Z^{(t)}\|^2 \leq
&3\EEE\left\| \frac{1}{R}\sum_{r=0}^{R-1}\widetilde{S}^{(t)}\nabla\widetilde{J}(W_{r}^{(t)})\right\|^2
+3\EEE\left\|\frac{\eta\widetilde{S}^{(t)}}{R}\sum_{r=0}^{R-1} \LL(W_r^{(t)}-W^{(t)})\right\|^2
\\
&+3\EEE\left\|\frac{\mut\eta\widetilde{S}^{(t)}\LL}{R}\sum_{r=0}^{R-1}\widetilde{S}^{(t)}\nabla\widetilde{F}(W_{r}^{(t)})\right\|^2.
\end{align}
Next, by Lemma~\ref{lemma:serverupdate},
Lemma~\ref{lemma:Jsmooth}\ref{lemma:Jsmooth_normGDJ2}, and the definition of $\EE^{(t)}$,
\begin{align}
\label{S1}
\nonumber
3\EEE\left\|\frac{1}{R}\sum_{r=0}^{R-1}\widetilde{S}^{(t)}\nabla\widetilde{J}(W_{r}^{(t)})\right\|^2
&\leq
\frac{3\tau}{R}\sum_{r=0}^{R-1} \EEE\left\|\nabla J(W_r^{(t)})\right\|^2
+ \frac{3\tau\sigma_1^2}{R}
\\
\nonumber
&\leq
\frac{3\tau}{R}\sum_{r=0}^{R-1} \left(2\beta_J^2
\EEE
\left\|
W_r^{(t)} - W^{(t)}
\right\|^2
+ 2\EEE
\left\|
\nabla J(W^{(t)})
\right\|^2\right)
+ \frac{3\tau\sigma_1^2}{R}
\\
&= 6\tau\beta_J^2\EE^{(t)}
+ 6\tau
\EEE\|\nabla J(W^{(t)})\|^2
+ \frac{3\tau\sigma_1^2}{R}.
\end{align}
It follows from Lemma~\ref{Sproperty}\ref{Sproperty_normSY} and then Lemma~\ref{Jensen} that
\begin{align}\label{S2}
\nonumber
3\EEE\left\|\frac{\eta\widetilde{S}^{(t)}}{R}\sum_{r=0}^{R-1} \LL(W_r^{(t)}-W^{(t)})\right\|^2
&=
3\tau\eta^2 \EEE\left\|\frac{1}{R}\sum_{r=0}^{R-1}\LL(W_r^{(t)}-W^{(t)})
\right\|^2
\\
\nonumber
&\leq
\frac{3\tau\eta^2}{R} \sum_{r=0}^{R-1}\EEE\left\|\LL(W_r^{(t)}-W^{(t)})\right\|^2
\\
& \leq
\frac{3\tau\eta^2}{R}\|\LL\|^2 \sum_{r=0}^{R-1}\EEE\left\|(W_r^{(t)}-W^{(t)})\right\|^2
= 3\tau\eta^2\rho^2\EE^{(t)}.
\end{align}
Now, using Lemma~\ref{Sproperty}\ref{Sproperty_norm} and then proceeding as in \eqref{S1}, we obtain that
\begin{align}
\label{S3}
\nonumber
3\EEE\left\|\frac{\mut\eta\widetilde{S}^{(t)}\LL}{R}\sum_{r=0}^{R-1}\widetilde{S}^{(t)}\nabla\widetilde{F}(W_{r}^{(t)})\right\|^2
& \leq
3\mut^2\eta^2\rho^2\EEE
\left\|\frac{1}{R}\sum_{r=0}^{R-1}\widetilde{S}^{(t)}\nabla\widetilde{F}(W_{r}^{(t)})\right\|^2
\\
&  \leq
\mut^2\eta^2\rho^2 \left(6\tau\beta^2\EE^{(t)}
+ 6\tau\EEE\|\nabla F(W^{(t)})\|^2
+ \frac{3\tau\sigma_1^2}{R}\right).
\end{align}
The proof is completed by combining \eqref{EZ2}--\eqref{S3}.
\end{proof}

\subsection{Convergence of \OurAlg and \dFed for Strongly Convex Cases (Proof of Theorem~\ref{theorem1})}
\label{prooftheoryconvex}

First, it follows from \eqref{Wnew} that
\begin{align}
\label{start1}
\EEE\|W^{(t+1)}-W^*\|^2
= \EEE\|W^{(t)}-W^*\|^2 -2\mut\EEE\langle
Z^{(t)}, W^{(t)}-W^*\rangle +\mut^2\EEE\| Z^{(t)}\|^2.
\end{align}
Let us now estimate the second term in the right hand side of \eqref{start1}.
Using Lemmas~\ref{Sproperty}\ref{Sproperty_ST},~\ref{lemma:Jsmoothconvex}\ref{lemma:Jsmoothconvex_3W},~\ref{Young}\ref{Young_XY}, and~\ref{Sproperty}\ref{Sproperty_normSY}, we have
\begin{align}\label{T1}
\nonumber
-2\mut\EEE\langle
Z^{(t)}, W^{(t)}-W^*\rangle &
=
\frac{2\mut \tau}{R}\sum_{r=0}^{R-1}\EEE\left\langle\nabla J(W_{r}^{(t)}), W^*-W^{(t)}\right\rangle
+ \frac{2\mut\tau \eta }{R}\sum_{r=0}^{R-1}\EEE\left\langle \LL (W_r^{(t)}-W^{(t)}), W^{(t)}-W^*\right\rangle
\\
\nonumber
&\quad+ \frac{2\mut^2\eta }{R}\sum_{r=0}^{R-1}\EEE\left\langle
\LL\widetilde{S}^{(t)}\nabla F(W_{r}^{(t)}), \widetilde{S}^{(t)}(W^{(t)}-W^*)\right\rangle
\\
\nonumber
& \leq
\frac{2\mut \tau}{R}\sum_{r=0}^{R-1}
\left(
\EEE[J(W^*) - J(W^{(t)})] - \frac{\alpha}{4}\EEE\|W^{(t)}-W^*\|^2 + \beta_J\EEE\|W_r^{(t)}-W^{(t)}\|^2
\right)
\\
\nonumber
&\quad
+ \frac{2\mut\tau\eta }{R}\sum_{r=0}^{R-1}
\left(
\frac{m}{2}\|\LL\|^2 \EEE\|W_r^{(t)}-W^{(t)}\|^2 + \frac{1}{2m}\EEE\|W^{(t)}-W^*\|^2
\right)
\\
\nonumber
& \quad+ \frac{2\mut^2\eta }{R}\sum_{r=0}^{R-1}
\left(
\frac{n}{2}\|\LL\|^2 \EEE\|\widetilde{S}^{(t)}\nabla F(W_r^{(t)})\|^2
+ \frac{1}{2n}\EEE\|\widetilde{S}^{(t)}(W^{(t)}-W^*)\|^2
\right)
\\
\nonumber
& \leq
\frac{2\mut \tau}{R}\sum_{r=0}^{R-1}
\left(
\EEE[J(W^*) - J(W^{(t)})] - \frac{\alpha}{4}\EEE\|W^{(t)}-W^*\|^2 + \beta_J\EEE\|W_r^{(t)}-W^{(t)}\|^2
\right)
\\
\nonumber
&\quad
+ \frac{2\mut\eta \tau}{R}\sum_{r=0}^{R-1}
\left(
\frac{m\rho^2}{2} \EEE\|W_r^{(t)}-W^{(t)}\|^2 + \frac{1}{2m}\EEE\|W^{(t)}-W^*\|^2
\right)
\\
& \quad+ \frac{2\mut^2\tau\eta }{R}\sum_{r=0}^{R-1}
\left(
\frac{n\rho^2}{2}\EEE\|\nabla F(W_r^{(t)})\|^2 + \frac{1}{2n}\EEE\|W^{(t)}-W^*\|^2
\right),
\end{align}
where $m,n> 0$ will be chosen later. In addition, Lemmas~\ref{lemma:Jsmooth}\ref{lemma:Jsmooth_normGDF2} and~\ref{lemma:Jsmoothconvex}\ref{lemma:Jsmoothconvex_GDF} imply that,
\begin{align}
\label{nablaFr}
\nonumber
\EEE\|\nabla F(W_r^{(t)})\|^2
&\leq
2\beta ^2\EEE\|W_r^{(t)}-W^{(t)}\|^2 +2\EEE\|\nabla F(W^{(t)})\|^2
\\
&\leq
2\beta ^2\EEE\|W_r^{(t)}-W^{(t)}\|^2
+ \frac{8\beta ^2}{\alpha}\EEE[J(W^{(t)})-J(W^*)] + 4\sigma_2^2.
\end{align}
By using \eqref{T1}, \eqref{nablaFr}, and the definition of $\EE^{(t)}$,
\begin{align}\label{}
\nonumber
T_1
&\leq -\tau\left(2\mut-\frac{8n\mut^2\eta\rho^2\beta ^2}{\alpha}\right)\EEE[J(W^{(t)})-J(W^*)]
- \tau\left(
\frac{\mut\alpha}{2}-\frac{\mut\eta}{m}
-\frac{\mut^2\eta }{n }
\right)
\EEE\|W^{(t)}-W^*\|^2
\\
&\quad + \tau\left(2\mut\beta_J + m\mut\eta\rho^2 + 2n\mut^2\eta\rho^2\beta ^2\right)\EE^{(t)}
+ 4n\mut^2\tau\eta\rho^2\sigma_2^2.
\end{align}
Setting $m =\frac{8\eta}{\alpha}$ and $n =\frac{8\mut\eta}{\alpha}$, we have
\begin{align}
\label{T1new}
\nonumber
T_1
&\leq  -\tau\left(2\mut-\frac{64\mut^3\eta^2\rho^2\beta ^2}{\alpha^2}\right) \EEE[J(W^{(t)})-J(W^*)]
- \frac{\mut\tau\alpha  }{4}\EEE\|W^{(t)}-W^*\|^2
\\
&\quad
+ \mut\tau
\left(
2\beta_J
+ \frac{8\eta^2\rho^2}{\alpha}
+ \frac{16\mut^2\eta^2\rho^2\beta ^2}{\alpha}
\right) \EE^{(t)}
+ \frac{32\mut^3\tau\eta^2\rho^2  \sigma_2^2}{\alpha}.
\end{align}
Combining this with \eqref{start1} and Lemma~\ref{lemma:EZ}, we get
\begin{align}
\nonumber
&\EEE\|W^{(t+1)}-W^*\|^2
\\
\nonumber
&\leq \left(1-\frac{\mut\tau\alpha }{4}\right) \EEE \|W^{(t)}-W^*\|^2
- \tau\left(
2\mut
-\frac{64\mut^3\eta^2\rho^2\beta^2}{\alpha^2}
\right)\EEE[J(W^{(t)})-J(W^*)]
\\
\nonumber
& \quad
+ \mut \tau
\left(
2\beta_J
+ \frac{8\eta^2\rho^2}{\alpha}
+ \frac{16\mut^2\eta^2\rho^2\beta ^2}{\alpha}
+ 6\mut\beta_J^2
+ 3\mut\eta^2\rho^2
+ 6\mut^3\eta^2\rho^2\beta^2
\right)
\EE^{(t)}
\\
\nonumber
& \quad+ 6\mut^2\tau \EEE\|\nabla J(W^{(t)})\|^2
+ 6\mut^4\tau\eta^2\rho^2\EEE\|\nabla F(W^{(t)})\|^2
+ \frac{32\mut^3\tau\eta^2\rho^2  \sigma_2^2}{\alpha}
+ \frac{3\mut^2\tau(1+\mut^2\eta^2\rho^2)\sigma_1^2}{R}
\\
\nonumber
&\leq \left(1-\frac{\mut\tau\alpha }{4}\right) \EEE \|W^{(t)}-W^*\|^2
- \tau\left(
2\mut
-\frac{64\mut^3\eta^2\rho^2\beta^2}{\alpha^2}
\right)\EEE[J(W^{(t)})-J(W^*)]+ \mut \tau p
\EE^{(t)}
\\
& \quad
+ 6\mut^2\tau \EEE\|\nabla J(W^{(t)})\|^2
+ 6\mut^4\tau\eta^2\rho^2\EEE\|\nabla F(W^{(t)})\|^2
+ \frac{32\mut^3\tau\eta^2\rho^2  \sigma_2^2}{\alpha}
+ \frac{3\mut^2\tau(1+\mut^2\eta^2\rho^2)\sigma_1^2}{R},
\end{align}
In what follows, we assume that \eqref{normalserverupdate} holds.
where we use \eqref{normalserverupdate} to estimate
\begin{align}
\nonumber
&2\beta_J
+ \frac{8\eta^2\rho^2}{\alpha}
+ \frac{16\mut^2\eta^2\rho^2\beta ^2}{\alpha}
+ 6\mut\beta_J^2
+ 3\mut\eta^2\rho^2
+ 6\mut^3\eta^2\rho^2\beta^2
\\
&\leq
p= 2\beta_J
+ \frac{8\eta^2\rho^2}{\alpha}
+ \frac{64\beta ^2}{\alpha}
+ \frac{12\beta_J^2}{\eta\rho}
+ 6\eta\rho
+ \frac{48\beta^2}{\eta\rho}.
\end{align}

Using Lemmas~\ref{lemma:boundeddrift},~\ref{lemma:Jsmoothconvex}\ref{lemma:Jsmoothconvex_GDJ}, and~\ref{lemma:Jsmoothconvex}\ref{lemma:Jsmoothconvex_GDF}, we have
\begin{align}
\nonumber
&\EEE\|W^{(t+1)}-W^*\|^2
\\
\nonumber
&\leq \left(1-\frac{\mut\tau\alpha }{4}\right) \EEE \|W^{(t)}-W^*\|^2
- \tau\left(
2\mut
-\frac{64\mut^3\eta^2\rho^2\beta^2}{\alpha^2}
\right)\EEE[J(W^{(t)})-J(W^*)]
\\
\nonumber
& \quad
+ \mut \tau p
\left(4\mut^2\tau\EEE\|\nabla F(W^{(t)})\|^2
+ \frac{2\mut^2\tau\sigma_1^2}{ R}\right)
\\
& \quad+ 6\mut^2\tau \EEE\|\nabla J(W^{(t)})\|^2
+ 6\mut^4\tau\eta^2\rho^2\EEE\|\nabla F(W^{(t)})\|^2
+ \frac{32\mut^3\tau\eta^2\rho^2  \sigma_2^2}{\alpha}
+ \frac{3\mut^2\tau(1+\mut^2\eta^2\rho^2)\sigma_1^2}{R}.
\\
\nonumber
& \leq \left(1-\frac{\mut\tau\alpha }{4}\right) \EEE \|W^{(t)}-W^*\|^2
- \tau\left(
2\mut
-\frac{64\mut^3\eta^2\rho^2\beta^2}{\alpha^2}
\right)\EEE[J(W^{(t)})-J(W^*)]
\\
\nonumber
& \quad
+ (4p\mut^3\tau^2+6\mut^4\tau\eta^2\rho^2)
\left(4\frac{\beta ^2}{\alpha}\EEE[J(W)-J(W^*)] + 2\sigma_2^2\right)
+ \frac{2p\mut^3\tau^2\sigma_1^2}{ R}
\\
& \quad+ 12\mut^2\tau \beta_J\EEE[J(W)-J(W^*)]
+ \frac{32\mut^3\tau\eta^2\rho^2  \sigma_2^2}{\alpha}
+ \frac{3\mut^2\tau(1+\mut^2\eta^2\rho^2)\sigma_1^2}{R}
\\
\nonumber
& \overset{\eqref{normalserverupdate}}{\leq}
\left(1-\frac{\mut\tau\alpha }{4}\right) \EEE \|W^{(t)}-W^*\|^2
- \tau\left[
2\mut - \mut^2\underbrace{\left(
\frac{128\eta\rho\beta^2}{\alpha^2}
+ 12\beta_J
+\frac{96\beta ^2}{\alpha}
+ \frac{32p\beta^2}{\alpha\eta\rho}
\right)}_{q}
\right]
\EEE[J(W^{(t)})-J(W^*)]
\\
&\qquad
+ \mut^3 \tau^2 \underbrace{\frac{p\left(8 R\sigma_2^2+ 2\sigma_1^2\right)}{R}}_{C_2}
+ \mut^2 \tau
\underbrace{
\frac{(64\eta\rho R+ 48)\sigma_2^2 + 15\sigma_1^2}{\alpha R}
}_{C_1}.
\end{align}

Let $\mu\leq \frac{\mut_1}{R}$. Then $\mut\leq \mut_1 = \min\left\{\frac{1}{q},\frac{2}{\eta\rho}\right\} \leq \frac{1}{q}$, which implies that $2\mut-\mut^2 q\geq \mut$, and so
\begin{align}\label{VarianceW2}
\EEE\|W^{(t+1)}-W^*\|^2
\leq \left(1-\frac{\mut\tau\alpha}{4}\right) \EEE\|W^{(t)}-W^*\|^2
- \mut \tau\EEE[J(W^{(t)})-J(W^*)] + \mut^3  \tau^2 C_2 + \mut^2  \tau C_1.
\end{align}


Recalling that $\Delta^{(t)} =\|W^{(t)}-W^*\|^2$, rearranging the terms, and multiplying both sides of \eqref{VarianceW2} with $\frac{\theta^{(t)}}{\mut \tau\Theta_T}$, where $\Theta_T=\sum_{t=0}^{T-1}\theta^{(t)}$, we obtain that
\begin{align}
\nonumber
\sum_{t=0}^{T-1}\frac{\theta^{(t)}\EEE [J(W^{(t)})]}{\Theta_T} - J(W^*)
&\leq
\sum_{t=0}^{T-1}\EEE
\left[
\left(1-\frac{\mut\tau\alpha}{4}\right)
\frac{\theta^{(t)}\Delta^{(t)}}{\mut \tau\Theta_T}
- \frac{\theta^{(t)}\Delta^{(t+1)}}{\mut \tau\Theta_T}
\right]
+ \mu^2  \tau C_2  + \mut C_1
\\
\label{telescope}
&= \sum_{t=0}^{T-1}\EEE
\left[\frac{\theta^{(t-1)}\Delta^{(t)}-\theta^{(t)}\Delta^{(t+1)}}{\mut \tau \Theta_T}\right] + \mu^2  \tau C_2  + \mut C_1
\\
\nonumber
& = \frac{1}{\mut \tau\Theta_T}\Delta^{(0)} - \frac{\theta^{(T-1)}}{\mut \tau\Theta_T}\EEE\Delta^{(T)} + \mut^2 \tau C_2 + \mut C_1
\\
\label{varianceJ}
& \leq \frac{1}{\mut \tau\Theta_T}\Delta^{(0)}  + \mut^2 \tau C_2  + \mut C_1.
\end{align}
Here, \eqref{telescope} follows from the fact that $\left(1-\frac{\mut\tau\alpha}{4}\right)\theta^{(t)}=\theta^{(t-1)}$ due to $\theta^{(t)}=\left(1-\frac{\mut\tau\alpha}{4}\right)^{-(t+1)}$. We then have
\begin{align}\label{ThetaT}
\nonumber
\Theta_T
&
= \sum_{t=0}^{T-1}\left(1-\frac{\mut\tau\alpha}{4}\right)^{-(t+1)}
\\
\nonumber
&
=\left(1-\frac{\mut\tau\alpha}{4}\right)^{-T}\sum_{t=0}^{T-1}\left(1-\frac{\mut\tau\alpha}{4}\right)^t
\\
\nonumber
&
=\left(1-\frac{\mut\tau\alpha}{4}\right)^{-T}
\frac{1-\left(1-\frac{\mut\tau\alpha}{4}\right)^T}{\frac{\mut\tau\alpha}{4}}.
\end{align}

Now, let $T \geq \frac{4N}{\mut_1\alpha S}$. Then $\left(1-\frac{\mut\tau\alpha}{4}\right)^T\leq \exp\left(-\frac{\mut\tau\alpha T}{4}\right)\leq\exp(-1)\leq\frac{3}{4}$, and thus
\begin{align}
\Theta_T \geq  \left(1-\frac{\mut\tau\alpha}{4}\right)^{-T} \frac{1}{\mut\tau\alpha} = \frac{\theta^{(T-1)}}{\mut\tau\alpha},
\end{align}
which yields $\frac{1}{\mut \tau\Theta_T}\leq \frac{\alpha}{\theta^{(T-1)}}
\leq\alpha e^{-\frac{\mut\tau\alpha T}{4}}$. Therefore, \eqref{varianceJ} becomes
\begin{align}
\label{varianceJ2}
\sum_{t=0}^{T-1}\frac{\theta^{(t)}\EEE [J(W^{(t)})]}{\Theta_T} - J(W^*)
& \leq \alpha\Delta^{(0)} e^{-\frac{\mut\tau\alpha T}{4}} + \mut^2 \tau C_2 + \mut C_1,
\end{align}
which together with the convexity of $J$ implies that
\begin{align}
\label{EvarianceJ}
\EEE\left[J(\widetilde{W}^{(T)})-J(W^*)\right] =\EEE\left[J\left(\sum_{t=0}^{T-1}\frac{\theta^{(t)}}{\Theta_T}W^{(t)}\right)\right]-J(W^*)
 \leq \alpha\Delta^{(0)} e^{-\frac{\mut\tau\alpha T}{4}}  + \mut^2 \tau C_2 + \mut C_1.
\end{align}
Following the same of approach in \cite{sai20ICML,arjevani20a,stich19,kulunchakov20}, we consider the following cases.
\begin{itemize}
  \item If $\mut_1 \geq \widehat{\mu}:=
  \max\left\{\frac{4}{\alpha \tau T},
  \frac{4}{\alpha \tau T}\log\left(\frac{\alpha^2 \tau \Delta^{(0)} T}{C_1}\right)\right\}$, then we choose $\mu=\widehat{\mu}$ and have
  \begin{align}\label{}
  \EEE\left[J(\widetilde{W}^{(T)})-J(W^*)\right]
 \leq \widetilde{\OO}\left(\frac{C_2}{\alpha^2 \tau T^2}\right)
 + \widetilde{\OO}\left(\frac{C_1}{\alpha \tau T}\right).
  \end{align}
  \item If $\frac{4}{\alpha \tau T}\leq\mut_1\leq \widehat{\mu}$,
      then we choose $\mu=\mut_1$ and have
  \begin{align}\label{}
  \EEE\left[J(\widetilde{W}^{(T)})-J(W^*)\right]
 \leq \OO\left(\alpha\Delta^{(0)} e^{-\frac{\mut_1\alpha \tau T}{4}}\right)
 + \widetilde{\OO}\left(\frac{C_2}{\alpha^2 \tau T^2}\right)
 + \widetilde{\OO}\left(\frac{C_1}{\alpha \tau T}\right).
  \end{align}
\end{itemize}
By combining the above two cases,
\begin{align}\label{}
\EEE\left[J(\widetilde{W}^{(T)})-J(W^*)\right]
\leq \widetilde{\OO}
\left(
\alpha\Delta^{(0)} e^{-\frac{\mut_1\alpha \tau T}{4}}
+ \frac{R\sigma_2^2+\sigma_1^2}{(\alpha T)^2RS}
+ \frac{R\sigma_2^2+\sigma_1^2}{\alpha TRS}
\right),
\end{align}
which implies \eqref{theorem1a}. The remaining conclusion directly follows from \eqref{theorem1a}.

\subsection{Convergence of \OurAlg and \dFed for Nonconvex Cases (Proof of Theorem~\ref{theorem2})}
By Lemma~\ref{lemma:Jsmooth}\ref{lemma:Jsmooth_distanceJ} and \eqref{Wnew},
\begin{align}
\label{start2}
\nonumber
\EEE \left[J(W^{(t+1)}) - J(W^{(t)})\right]
&\leq
\EEE\langle \nabla J(W^{(t)}),W^{(t+1)}-W^{(t)} \rangle
+ \frac{\beta_J}{2}\EEE\|W^{(t+1)}-W^{(t)}\|^2
\\
\nonumber
&= - \EEE
\left\langle
\nabla J(W^{(t)}),Z^{(t)}
\right\rangle
+ \frac{\mut^2\beta_J}{2}\EEE\|Z^{(t)}\|^2
\\
\nonumber
& =
\underbrace{-\EEE\left\langle
\nabla J(W^{(t)}),\frac{\mut \tau}{R}\sum_{r=0}^{R-1}\nabla J(W_{r}^{(t)})
\right\rangle}_{T_2}
+ \underbrace{\EEE\left\langle
\nabla J(W^{(t)}),\frac{\mut\eta \tau\LL}{R}\sum_{r=0}^{R-1}(W_r^{(t)}-W^{(t)})
\right\rangle}_{T_3}
\\
&\quad
+ \underbrace{\EEE\left\langle
\nabla J(W^{(t)}),\frac{\mut^2\eta \widetilde{S}^{(t)}\LL}{R}\sum_{r=0}^{R-1}\widetilde{S}^{(t)}\nabla F(W_{r}^{(t)})
\right\rangle}_{T_4}
+ \frac{\mut^2\beta_J}{2}\EEE\|Z^{(t)}\|^2.
\end{align}

Using the fact that $-xy \leq \frac{-2xy+y^2}{2} = \frac{-x^2 + (y-x)^2}{2}, \forall x,y \in\RRR$, then Lemma~\ref{Jensen} and Lemma~\ref{lemma:Jsmooth}\ref{lemma:Jsmooth_distanceGDJ}, we have
\begin{align}\label{}
\nonumber
T_2
&\leq
- \frac{\mut \tau}{2}\EEE \|\nabla J(W^{(t)})\|^2
+ \frac{\mut \tau}{2} \EEE\left\|\frac{1}{R}\sum_{r=0}^{R-1}\nabla J(W_r^{(t)})-\nabla J(W^{(t)})\right\|^2
\\
\nonumber
& \leq
- \frac{\mut \tau}{2}\EEE \|\nabla J(W^{(t)})\|^2
+ \frac{\mut \tau}{2R}\sum_{r=0}^{R-1} \EEE\left\|\nabla J(W_r^{(t)})-\nabla J(W^{(t)})\right\|^2
\\
\nonumber
& \leq
- \frac{\mut \tau}{2}\EEE \|\nabla J(W^{(t)})\|^2
+ \frac{\mut \tau\beta_J^2}{2R}\sum_{r=0}^{R-1} \EEE\|W_r^{(t)}-W^{(t)}\|^2
\\
& =
- \frac{\mut \tau}{2}\EEE \|\nabla J(W^{(t)})\|^2
+ \frac{\mut \tau\beta_J^2}{2}\EE^{(t)}.
\end{align}

For the terms $T_3$, by Lemmas~\ref{Young}\ref{Young_XY}, and~\ref{Jensen},
\begin{align}\label{T3}
\nonumber
T_3
&\leq \frac{z\mut \tau}{2}\EEE \|\nabla J(W^{(t)})\|^2
+ \frac{\mut \tau}{2z}\EEE \left\|\frac{\eta\LL}{R}\sum_{r=0}^{R-1}(W_r^{(t)}-W^{(t)})\right\|^2
\\
\nonumber
& \leq
\frac{z\mut \tau}{2}\EEE \|\nabla J(W^{(t)})\|^2
+ \frac{\mut \tau\eta^2}{2zR}\sum_{r=0}^{R-1}\|\LL\|^2\EEE\|W_r^{(t)}-W^{(t)}\|^2
\\
&= \frac{z\mut \tau}{2}\EEE \|\nabla J(W^{(t)})\|^2
+ \frac{\mut \tau\eta^2\rho^2}{2z}\EE^{(t)},
\end{align}
where $z >0$ will be chosen later.

For the terms $T_4$, using Lemmas~\ref{Sproperty}\ref{Sproperty_ST},~\ref{Young}\ref{Young_XY},~\ref{Jensen},~\ref{Sproperty}\ref{Sproperty_normSY},~\ref{lemma:Jsmooth}\ref{lemma:Jsmooth_normGDJ2},
we derive that
\begin{align}\label{T4first}
\nonumber
T_4
& =
\mut\EEE\left\langle
\widetilde{S}^{(t)}\nabla J(W^{(t)}),\frac{\mut\eta \LL}{R}\sum_{r=0}^{R-1}\widetilde{S}^{(t)}\nabla F(W_{r}^{(t)})
\right\rangle
\\
\nonumber
& \leq
\frac{s\mut }{2}\EEE \|\widetilde{S}^{(t)}\nabla J(W^{(t)})\|^2
+ \frac{\mut^3\eta^2}{2s}\EEE\left\|\frac{\LL}{R}\sum_{r=0}^{R-1}
\widetilde{S}^{(t)}\nabla F(W_r^{(t)})
\right\|^2
\\
\nonumber
& \leq
\frac{s\mut }{2}\EEE \|\widetilde{S}^{(t)}\nabla J(W^{(t)})\|^2
+ \frac{\mut^3\eta^2 }{2sR}\sum_{r=0}^{R-1}\|\LL\|^2\EEE
\left\|
\widetilde{S}^{(t)}\nabla F(W_r^{(t)})
\right\|^2
\\
\nonumber
&\leq
\frac{s\mut \tau}{2}\EEE \|\nabla J(W^{(t)})\|^2
+ \frac{\mut^3\tau\eta^2 \rho^2}{2sR}
\sum_{r=0}^{R-1}\EEE
\|\nabla F(W_r^{(t)})\|^2
\\
\nonumber
& \leq
\frac{s\mut \tau}{2}\EEE \|\nabla J(W^{(t)})\|^2
+ \frac{\mut^3\tau\eta^2 \rho^2}{2sR}
\sum_{r=0}^{R-1}
\left(2\beta^2\EEE\|W_r^{(t)}- W^{(t)}\|^2 +2\|\nabla F(W^{(t)})\|^2\right),
\\
& = \frac{s\mut \tau}{2}\EEE \|\nabla J(W^{(t)})\|^2
+ \frac{\mut^3\tau\eta^2\rho^2\beta^2}{s}\EE^{(t)}
+\frac{\mut^3\tau\eta^2 \rho^2}{s}\|\nabla F(W^{(t)})\|^2,
\end{align}
where $s >0$ will be chosen later.
By the definition of $\sigma_2$ in Lemma~\ref{lemma:Boundedgradient}, it holds that $\|\nabla F(W)\|^2 \leq \sigma_2^2 + \|\nabla J(W)\|^2$, which together with \eqref{normalserverupdate} and \eqref{T4first} yields
\begin{align}\label{T4}
\nonumber
T_4 & \leq
\frac{s\mut \tau}{2}\EEE\|\nabla J(W^{(t)})\|^2
+ \frac{\mut^3\tau\eta^2 \rho^2}{sR}
\sum_{r=0}^{R-1}
\left(
\beta ^2\EEE\|W_r^{(t)}- W^{(t)}\|^2
+ \EEE\|\nabla J(W^{(t)}) \|^2 + \sigma_2^2
\right)
\\
& \leq
\mut \tau\left(\frac{s}{2} + \frac{2\mut\eta\rho}{s}\right)
\EEE\|\nabla J(W^{(t)})\|^2
+ \frac{4\mut\tau\beta ^2}{s}\EE^{(t)}
+ \frac{2\mut^2\tau\eta\rho \sigma_2^2}{sR},
\end{align}

Choosing $z=s=\frac{1}{4}$, combining \eqref{start2}--\eqref{T4first} with Lemma~\ref{lemma:EZ}, and then using Lemma~\ref{lemma:boundeddrift}, we obtain that
\begin{align}\label{start3}
\nonumber
\EEE[J(W^{(t+1)}) - J(W^{(t)})]
&\leq -\tau\left(\frac{\mut}{4}-3\mut^2\beta_J\right)
\EEE\|\nabla J(W^{(t)})\|^2
+\left(4\mut^3\tau\eta^2 \rho^2
+ 3\mut^4\tau\eta^2\rho^2\beta_J
\right)
\|\nabla F(W^{(t)})\|^2
\\
\nonumber
&\quad
+ \mut\tau \left(\frac{\beta_J^2}{2}
+ 2\eta^2\rho^2
+ 4\mut^2\eta^2\rho^2\beta^2
+ 3\mut\beta_J^3+\frac{3}{2}\mut\eta^2\rho^2\beta_J+3\mut^3\eta^2\rho^2\beta_J\beta^2\right)\EE^{(t)}
\\
\nonumber
& \quad
+ \frac{3\mut^2\tau(1+\mut^2\eta^2\rho^2)\beta_J\sigma_1^2}{2R}
\\
\nonumber
&\leq
-\tau\left(\frac{\mut}{4}-3\mut^2\beta_J\right)
\EEE\|\nabla J(W^{(t)})\|^2
+\left(4\mut^3\tau\eta^2 \rho^2
+ 3\mut^4\tau\eta^2\rho^2\beta_J
\right)
\|\nabla F(W^{(t)})\|^2
\\
&\quad
+ \mut\tau u \left(4\mut^2\tau\EEE\|\nabla F(W^{(t)})\|^2
+ \frac{2\mut^2\tau\sigma_1^2}{ R}\right)
+ \frac{3\mut^2\tau(1+\mut^2\eta^2\rho^2)\beta_J\sigma_1^2}{2R},
\end{align}
where we use \eqref{normalserverupdate} to estimate
\begin{align}
\nonumber
&\frac{\beta_J^2}{2}
+ 2\eta^2\rho^2
+ 4\mut^2\eta^2\rho^2\beta^2
+ 3\mut\beta_J^3+\frac{3}{2}\mut\eta^2\rho^2\beta_J+3\mut^3\eta^2\rho^2\beta_J\beta^2 \\
&\leq u =\frac{\beta_J^2}{2} + 2\eta^2\rho^2 + 16\beta ^2
+ \frac{6\beta_J^3}{\eta\rho} + 3\eta\rho\beta_J + \frac{24\beta_J\beta ^2}{\eta\rho}.
\end{align}

By the definition of $\sigma_2$ in Lemma~\ref{lemma:Boundedgradient}, it holds that $\|\nabla F(W)\|^2 \leq \sigma_2^2 +\|\nabla J(W)\|^2$, which together with \eqref{start3} and \eqref{normalserverupdate} yields
\begin{align}
\nonumber
\EEE [J(W^{(t+1)}) - J(W^{(t)})]
&\leq -\tau\left(\frac{\mut}{4}-3\mut^2\beta_J\right)
\EEE\|\nabla J(W^{(t)})\|^2
+\left(4\mut^3\tau\eta^2 \rho^2
+ 3\mut^4\tau\eta^2\rho^2\beta_J
\right)
\left(\sigma_2^2 +\EEE\|\nabla J(W)\|^2\right)
\\
\nonumber
&\quad
+ \mut^3\tau^2 u \left[4\left(\sigma_2^2 +\EEE\|\nabla J(W)\|^2\right)
+ \frac{2\sigma_1^2}{ R}\right]
+ \frac{3\mut^2\tau(1+\mut^2\eta^2\rho^2)\beta_J\sigma_1^2}{2R}
\\
\nonumber
&\overset{ \tau\leq 1}{\leq}
-\tau\left(
\frac{\mut}{4}
-8\mut^2\tau\eta \rho
-3\mut^2\beta_J
- 3\mut^2\tau\beta_J
- \frac{8\mut^2 u}{\eta\rho}
\right)
\EEE\|\nabla J(W^{(t)})\|^2
\\
&\quad +\mut^3 \tau^2\underbrace{ \frac{u\left(4R\sigma_2^2+ 2\sigma_1^2\right)}{ R}}_{C_4}
+ \mut^2 \tau
\underbrace{
\frac{(8\eta\rho+ 12\beta_J)R\sigma_2^2+27\beta_J\sigma_1^2}{2R}
}_{C_3}
\end{align}


Now, let
\begin{align}\label{tildemu}
\mut\leq \mut_2 = \min\left\{\frac{2}{\eta\rho}, \frac{1}{v}\right\}=
\min\left\{\frac{2}{\eta\rho}, \frac{1}{8\left(8\eta\rho+3\beta_J+12\beta_J+\frac{8u}{\eta\rho}\right)}\right\}.
\end{align}
Then $-\tau\left(
\frac{\mut}{4}-8\mut^2\eta\rho- 3\mut^2\beta_J - 12\mut^2\beta_J - \frac{8\mut^2 u}{ \eta\rho}
\right)\leq - \frac{\mut \tau}{8}$, and so
\begin{align}\label{distanceJ2}
\EEE \left[J(W^{(t+1)}) - J(W^{(t)})\right]
\leq  -\frac{\mut \tau}{8}\EEE\|\nabla J(W^{(t)})\|^2
+ \mut^3 \tau^2C_4
+ \mut^2 \tau C_3.
\end{align}
By re-arranging the terms of \eqref{distanceJ2} and telescoping, we have
\begin{align}\label{}
\nonumber
\frac{1}{8T}\sum_{t=0}^{T-1}\EEE\|\nabla J(W^{(t)})\|^2
&\leq \frac{\EEE\left[J(W^{(0)})-J(W^{(T)})\right]}{\mu \tau T}
+ \mut^2 \tau C_4
+ \mut C_3
\\
&\leq \frac{\EEE\left[J(W^{(0)})-J(W^*)\right]}{\mu \tau T}
+ \mut^2 \tau C_4
+ \mut C_3.
\end{align}
Let $\Delta_J :=J(W^{(0)})-J(W^*)$. Following the same of approach in \cite{sai20ICML,arjevani20a,stich19,kulunchakov20}, we consider the following cases.
\begin{itemize}
  \item If $\mut_2^3\leq \frac{\Delta_J}{C_4 \tau^2 T}$ and $\mut_2^2\leq \frac{\Delta_J}{C_3 \tau T}$, then we choose $\mut=\mut_2$ to get
  \begin{align}\label{}
  \frac{1}{8T}\sum_{t=0}^{T-1}\EEE\|\nabla J(W^{(t)})\|^2
    \leq \frac{\Delta_J}{\mut_2 \tau T}
    + \frac{\Delta_J^{\frac{2}{3}} C_4^{\frac{1}{3}}}{\tau^{\frac{1}{3}}T^{\frac{2}{3}}}
    + \frac{\Delta_J^{\frac{1}{2}}C_3^{\frac{1}{2}}}{ \tau^{\frac{1}{2}}T^{\frac{1}{2}}}.
  \end{align}
  \item If $\mut_2^3\geq \frac{\Delta_J}{C_4 \tau^2 T}$ or $\mut_2^2\geq \frac{\Delta_J}{C_3 \tau T}$, then we choose $\mut=\min\left\{\left(\frac{\Delta_J}{C_4 \tau^2 T}\right)^{\frac{1}{3}},\left(\frac{\Delta_J}{C_5 \tau T}\right)^{\frac{1}{2}}\right\}$ to get
  \begin{align}\label{}
  \frac{1}{8T}\sum_{t=0}^{T-1}\EEE\|\nabla J(W^{(t)})\|^2
 \leq \frac{\Delta_J^{\frac{2}{3}} C_4^{\frac{1}{3}}}{\tau^{\frac{1}{3}}T^{\frac{2}{3}}}
    + \frac{\Delta_J^{\frac{1}{2}}C_3^{\frac{1}{2}}}{ \tau^{\frac{1}{2}}T^{\frac{1}{2}}}.
  \end{align}
\end{itemize}
Combining two cases, and with $t^*$ uniformly sampled from $\{0,\dots,T-1\}$, we have
\begin{align}\label{}
\nonumber
\frac{1}{T}\sum_{t=0}^{T-1}\EEE\|\nabla J(W^{(t)})\|^2
= \EEE\|\nabla J(W^{(t^*)})\|^2
&\leq \OO\left(\frac{\Delta_J}{\mut_2 \tau T}
    + \frac{\Delta_J^{\frac{2}{3}} C_4^{\frac{1}{3}}}{\tau^{\frac{1}{3}}T^{\frac{2}{3}}}
    + \frac{\Delta_J^{\frac{1}{2}}C_3^{\frac{1}{2}}}{ \tau^{\frac{1}{2}}T^{\frac{1}{2}}}\right)
\\
& = \OO\!\left(\frac{\Delta_J}{TS}
    \!+\! \frac{\Delta_J^{\frac{2}{3}} M^{\frac{2}{3}}}{T^{\frac{2}{3}}(RS)^{\frac{1}{3}}}
    \!+\! \frac{\Delta_J^{\frac{1}{2}}M^2}{\sqrt{TRS}}\right).
\end{align}
This proves \eqref{theorem2a}, which, in turn, implies \eqref{theorem2b}.

%% file: experiments_appendix.tex
\subsection{Additional Experimental Settings and Results}


\subsubsection{Statistics of All Datasets}
 We use four real datasets for the experiments including Human Activity Recognition, Vehicle sensor, MNIST, and CIFAR-10. The detailed statistics of all datasets are summarized in Table.~\ref{T:statistics}. 
\begin{table}[h]
	\centering
	\setlength\tabcolsep{3.5 pt} 
	\caption{  Statistics of all datasets using in the experiment.} 
	\begin{tabular}{l|cccll}
		\multirow{2}{*}{Dataset} & \multirow{2}{*}{ $N$} 
		&\multicolumn{1}{c}{\multirow{2}{*}{\begin{tabular}[c]{@{}c@{}}Total\\samples\end{tabular}}}
		&\multicolumn{1}{c}{\multirow{2}{*}{\begin{tabular}[c]{@{}c@{}}Num labels \\ / client\end{tabular}}}
		&\multicolumn{2}{l}{Samples / client}   \\ 
		& &                                &                              & \multicolumn{1}{c}{Mean} & \multicolumn{1}{c}{Std}\\ \hline
		\begin{tabular}[c]{@{}c@{}} Human\\ Activity\end{tabular}           & 30                             &     10,299        &      6            &              \multicolumn{1}{c}{343}  &   \multicolumn{1}{c}{35.1   }          \\ \hline
		\begin{tabular}[c]{@{}c@{}} Vehicle\\ Sensor\end{tabular}          & 23                             &             48,303      &    2        &         \multicolumn{1}{c}{2,100 }&  \multicolumn{1}{c}{380.5    }                   \\ \hline
		MNIST                     & 100                            &      61,866          &     2          &          \multicolumn{1}{c}{619  }       & \multicolumn{1}{c}{343.8}          \\ \hline
		CIFAR-10                  & 20                             &       54,572         &         3     &                \multicolumn{1}{c}{2729 }&\multicolumn{1}{c}{ 851.4}  \\ \hline           
	\end{tabular}
	\label{T:statistics}
\end{table}
\subsubsection{Learning Tasks at Local Clients}
\begin{itemize}
	\item \textbf{Strongly convex setting}: We use a multinomial logistic regression model (MLR) with a cross-entropy loss function and a $L_2$-regularization term for all strongly convex experiments on Human Activity, Vehicle Sensor, and MNIST datasets. The loss function at each client is defined as follow:
	\begin{align}
		F_k (w)  = \frac{-1}{D_k}{\sum_{j=1}^{D_k}\sum_{c=1}^{C}1_{\{y_j = c\}} \log\frac{\exp(\innProd{a_j, w_{c}})}{\sum_{i=1}^{C}\exp(\innProd{a_i, w_{i}})}}  +\frac{\alpha}{2}\sum_{c=1}^{C}\norm{w_c}_2^2. \nonumber		
	\end{align}
	
	\item \textbf{Nonconvex setting}: We use a simple DNN with one hidden layer, a ReLU activation function, and a softmax layer at the end of the network for Human Activity and Vehicle Sensor datasets. The size of hidden layer is 100 for Human Activity and 20 for Vehicle Sensor. In the case of MNIST, we use DNN with 2 hidden layers and both layers have the same size of 100.  For CIFAR-10, we follow the CNN structure of \cite{mcmahan17}. 
\end{itemize}
\begin{table}[t!]
	\centering
	\caption{ Performance comparison in mult-task setting without down-sampling data (All tasks participate, $R = 5$, $S=N$, mini-batch size $B = 20$, $\eta = 10^{-2}$, $T = 200$).}
	\begin{tabular}{l|l|ll}
		\multicolumn{1}{l|}{\multirow{2}{*}{Dataset} } & \multirow{2}{*}{Algorithm}  &  \multicolumn{2}{c}{Test accuracy}  \\ 
		&                                             & \multicolumn{1}{c}{Convex}        & \multicolumn{1}{c}{Nonconvex }   \\ \hline
		\multirow{4}{*}{\begin{tabular}[c]{@{}c@{}} Human\\ Activity\end{tabular}}       
		& \OurAlg              &  $\textbf{99.10}      \pm 0.18$  &  $\textbf{99.21}      \pm 0.15$                   \\
		& MOCHA         &          $ 98.79      \pm 0.04 $              \\
		& Local              &          $ 98.29    \pm 0.01$     &          $ 98.34    \pm 0.03$                \\
		& Global              		&  $93.79           		\pm 0.27$          &          $ 94.58    \pm 0.16$             \\ \hline
		\multirow{4}{*}{\begin{tabular}[c]{@{}c@{}} Vehicle\\ Sensor\end{tabular}  }       
		& \OurAlg          &        $ \textbf{91.16}   \pm 0.02$                &        $ \textbf{95.43}   \pm 0.09$           \\
		& MOCHA         &    $     90.94      \pm 0.05$        &           \\
		& Local           &          $ 88.16    \pm 0.05 $          &          $ 92.10    \pm 0.06 $             \\
		& Global          &  $80.21  \pm 0.12 $            &  $83.00  \pm 0.11 $                \\ \hline
		\multirow{4}{*}{\begin{tabular}[c]{@{}c@{}} MNIST\end{tabular}  }       
		& \OurAlg          &        $ \textbf{98.07}   \pm 0.02$            &        $ \textbf{98.61}   \pm 0.02$     \\
		& MOCHA         &     $   97.99  \pm 0.02 $                &    \\       
		& Local           &          $   97.95  \pm  0.01$       &    $97.99        \pm 0.02$             \\
		& Global          &  $92.04  \pm 0.02 $            &   $96.19  \pm 0.09		 $         \\ \hline
	\end{tabular}
	\label{T:compare_task1}	
\end{table}

\begin{table}[t!]
	\centering
	\caption{ Performance comparison of centralized setting without down-sampling data ($R = 5$, $ S = 0.1 N$, $B = 20$, $T = 200$).}
	\begin{tabular}{l|c|ll}
		\multirow{3}{*}{Dataset} & \multirow{3}{*}{Algorithm} & \multicolumn{2}{c}{Test Accurancy}                                                                                        \\
		&                            & \multicolumn{1}{c}{Convex} & \multicolumn{1}{c}{Non Convex} \\ \hline
		\multirow{4}{*}{CIFAR-10 }      
		& \OurAlg      &                &     $\textbf{79.40}    \pm 0.25 $                       \\
		& pFedMe    &             &    $    78.70 \pm 0.15 $                            \\
		& Per-FedAvg        &           &           $67.61     \pm 0.03 $                                   \\ 
		& FedAvg        &           		&  $36.32          		\pm 5.57$                     \\ \hline
		\multirow{4}{*}{MNIST}      
		& \OurAlg                  & $ \textbf{97.82}      \pm 0.02$                & $ 98.44     \pm 0.02$        \\
		& MOCHA      &       $ 97.80 \pm 0.02  $               \\	
		& pFedMe             &          $ 95.38    \pm 0.09$            &$\textbf{99.04} \pm 0.02$      \\
		& Per-FedAvg               &       $ 91.77     \pm 0.23  $ &       $ 97.59     \pm 0.30  $                \\
		& FedAvg            		&  $90.14         		\pm 0.61$     	&  $90.74         		\pm 1.62$                  \\ \hline
		\multirow{4}{*}{\begin{tabular}[c]{@{}c@{}} Vehicle\\ Sensor\end{tabular}  }       
		& \OurAlg             &        $ \textbf{89.84}   \pm 0.06$     &        $ \textbf{94.18}   \pm 0.08$          \\
		& MOCHA      &    $  89.73  \pm 0.89 $              	 \\
		& pFedMe         &            $85.87 \pm 0.02$                 &            $92.23 \pm 0.17$   		\\
		& Per-FedAvg           &         $82.21 \pm 0.01$    &         $87.50 \pm 1.21$               \\
		& FedAvg          &  $81.54  \pm 0.03 $    &  $85.61  \pm 0.07 $           \\ \hline
		\multirow{4}{*}{\begin{tabular}[c]{@{}c@{}} Human\\ Activity\end{tabular}}       
		& \OurAlg               &  $\textbf{97.75}      \pm 0.21$   &  $ \textbf{97.85}      \pm 0.39$      	 \\
		& MOCHA      &       $ 97.69    \pm 0.03 $               	\\
		& pFedMe             &          $ 97.52      \pm 0.09$         &          $ 97.60      \pm 0.09$             		\\
		& Per-FedAvg              &       $ 96.04      \pm 0.36  $    &       $ 96.21      \pm 0.33  $               \\
		& FedAvg             		&  $95.58           		\pm 0.05$  	&  $94.84           		\pm 0.07$           \\ \hline
	\end{tabular}
	\label{T:compare_performance2}	
\end{table}
\subsubsection{Performance of FedU in Federated Multi-Task Learning without down-sampling data}

The result in Table.~\ref{T:compare_task1}  shows that \OurAlg still achieves the highest performance, however,  the performance gaps between \OurAlg, MOCHA, and Local model are less appreciable. When the local data at a client is large enough, the Local model at one client can learn individually without contributions from others. Therefore, both \OurAlg and MOCHA will show advantages compared to the Local model in federated settings when there are various clients having a small number of data. 

\subsubsection{Comparison with Personalized Federated Learning Algorithms without down-sampling data}
Similar to the down-sampling data setting,  \OurAlg almost maintains the top performance in all scenarios showing in Table.~\ref{T:compare_performance2}. Only in the nonconvex case on MNIST, pFedMe performs slightly better than \OurAlg.

\subsubsection{Effect of non-i.i.d levels}
{To show the effect of different degrees of non-i.i.d  on \OurAlg, we did experiments on the MNIST dataset in Fig.~\ref{F:Mnist_effect} for illustration purposes. We use the Dirichlet distribution to generate MNIST non-i.i.d dataset with 100 clients following setting in \cite{hsuMeasuringEffectsNonIdentical2019,reddiADAPTIVEFEDERATEDOPTIMIZATION2021,wangTacklingObjectiveInconsistency2020}. Specifically, client's data is partitioned  by using Dirichlet distribution Dir$_N(\alpha)$, where $N=100$ is total number of clients and $\alpha$ is concentration parameter. In this setting, $\alpha  > 0$  is to control the identicalness among clients. When $\alpha \rightarrow \infty$, all clients have identical distributions, corresponding to the i.i.d setting. By contrast, when $\alpha \rightarrow 0$, each client holds examples from only one class chosen randomly \cite{hsuMeasuringEffectsNonIdentical2019}. In our experiment, we consider three different values for $\alpha \in  \{0.01, 0.1, 10.0\}$ to generate populations covering a spectrum of identicalness.
}
{In Fig.~\ref{F:Mnist_effect}, we consider the network of 100 clients and increase the level of non-i.i.d from left to right. When the level of non-i.i.d increases, personalized algorithms are more stable than traditional federated learning algorithms like FedAvg or FedProx. Importantly, our proposed algorithm \OurAlg performs well in all settings compared to other algorithms and is not much affected by the high degree of non-i.i.d.
}
\begin{figure*}[!t]
	\centering
	\begin{subfigure}{1\textwidth} 
		\centering
		{\includegraphics[scale=0.39]{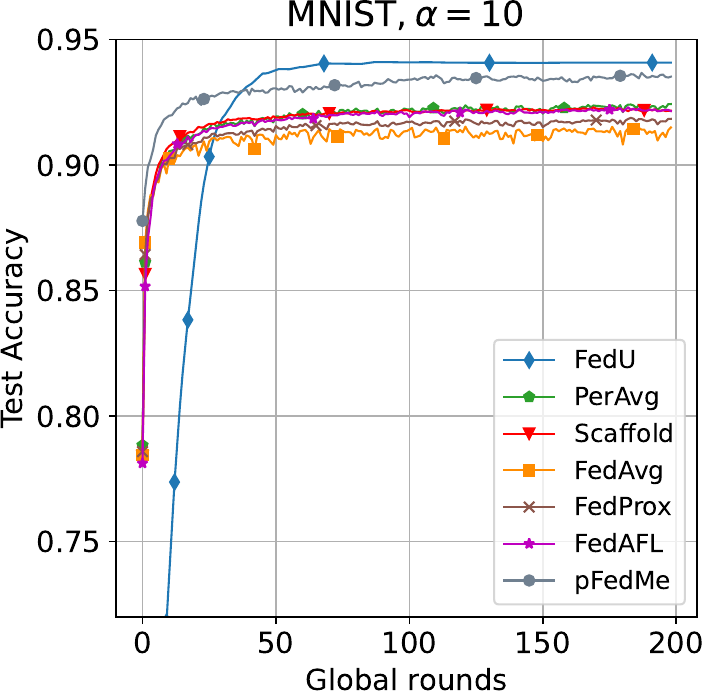}}\quad
		{\includegraphics[scale=0.39]{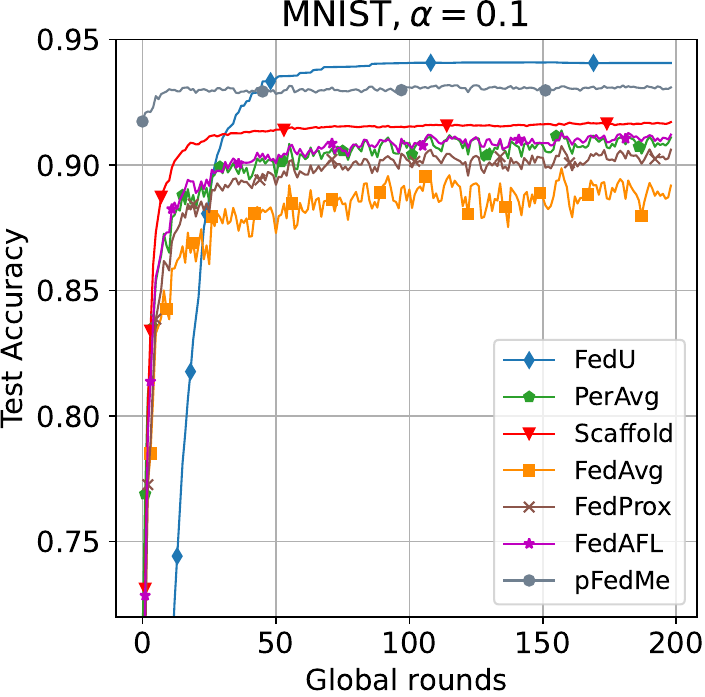}}\quad
		{\includegraphics[scale=0.39]{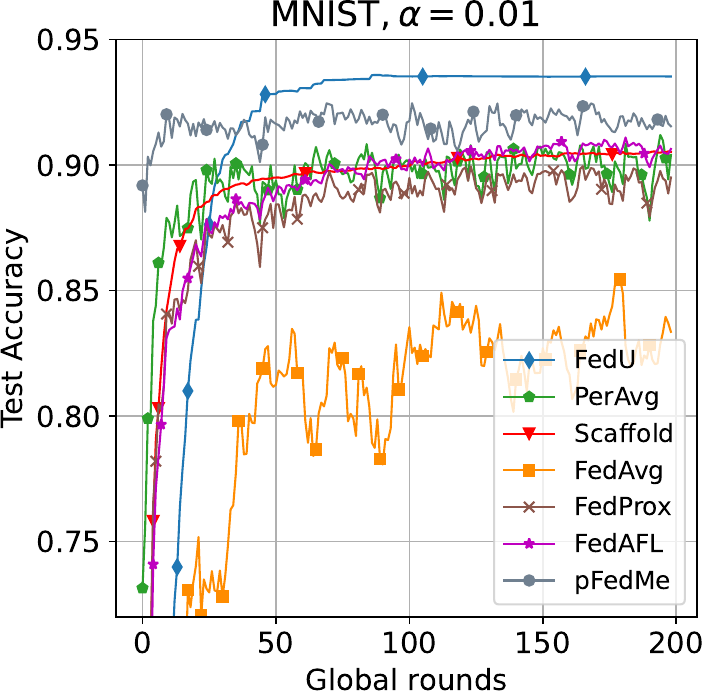}}
		{\includegraphics[scale=0.4]{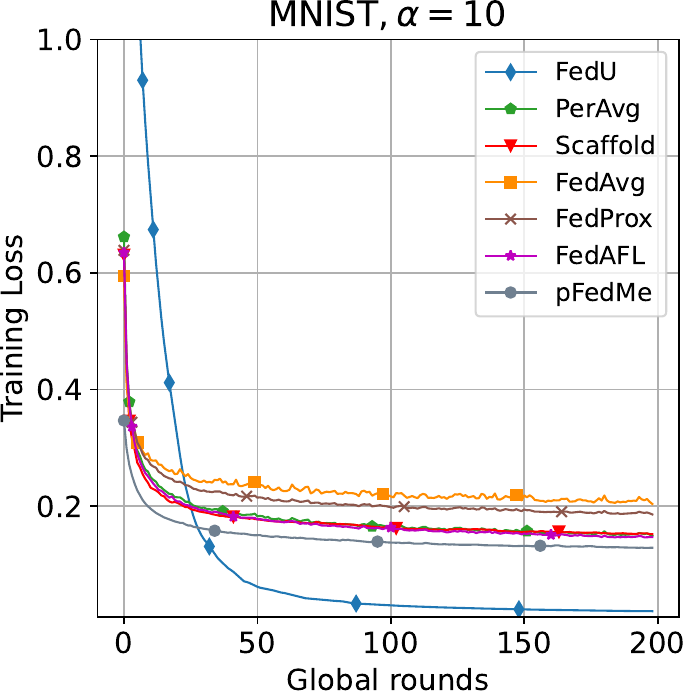}}\quad
		{\includegraphics[scale=0.4]{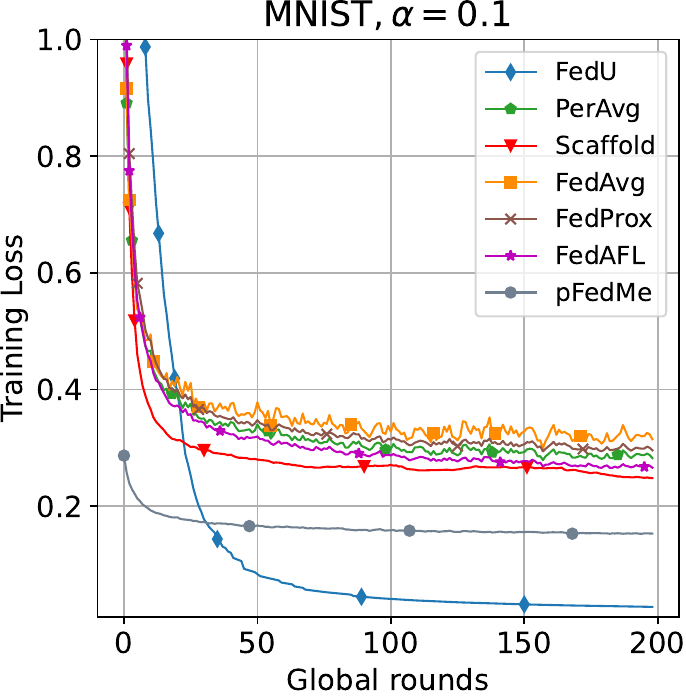}}\quad
		{\includegraphics[scale=0.4]{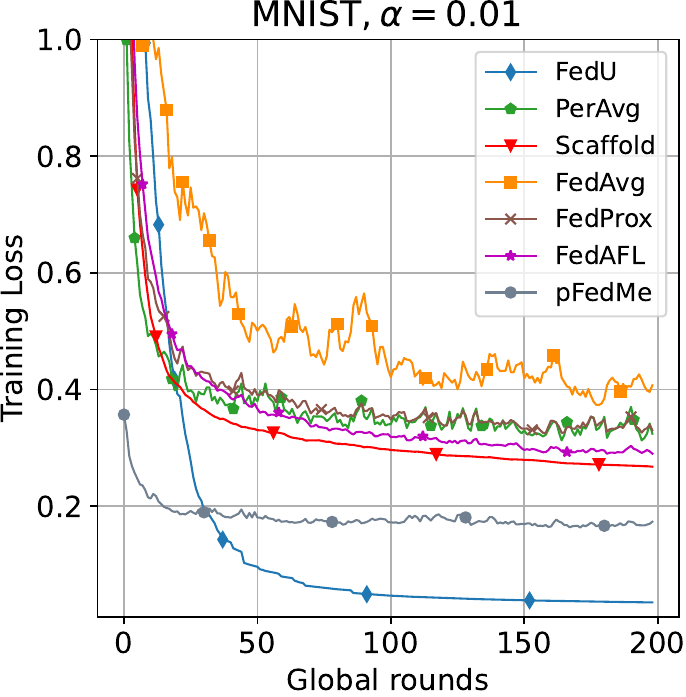}}
	\end{subfigure}
	\caption{{Effect of different non-i.i.d degrees on Federated Learning algorithms. $\alpha$ is the concentration parameter to control the level of non-i.i.d.}} \label{F:Mnist_effect}
\end{figure*}
%